\documentclass{colt2016}

\title[Solving Combinatorial Games using Products, Projections and Lexicographically Optimal Bases]{Solving Combinatorial Games using\\ Products, Projections and Lexicographically Optimal Bases}
\usepackage{times}

\SetKwInput{Input}{Input}
\SetKwInput{Output}{Output}
\usepackage{amsmath} 
\usepackage{amsfonts}
\usepackage{comment}
\usepackage{algorithm}
\usepackage{color}
 \coltauthor{\Name{Swati Gupta} \Email{swatig@mit.edu}\\
 \Name{Michel Goemans} \Email{goemans@math.mit.edu}\\
 \Name{Patrick Jaillet} \Email{jaillet@mit.edu}\\
 }

\begin{document}

\maketitle

\begin{abstract} In order to find Nash-equilibria for two-player zero-sum games where each player plays combinatorial objects like spanning trees, matchings etc, we consider two online learning algorithms: the online mirror descent (OMD) algorithm and the multiplicative weights update (MWU) algorithm. The OMD algorithm requires the computation of a certain Bregman projection, that has closed form solutions for simple convex sets like the Euclidean ball or the simplex. However, for general polyhedra one often needs to exploit the general machinery of convex optimization. We give a novel primal-style algorithm for computing Bregman projections on the base polytopes of polymatroids. Next, in the case of the MWU algorithm, although it scales logarithmically in the number of pure strategies or experts $N$ in terms of regret, the algorithm takes time polynomial in $N$; this especially becomes a problem when learning combinatorial objects. We give a general recipe to simulate the multiplicative weights update algorithm in time polynomial in their natural dimension. This is useful whenever there exists a polynomial time generalized counting oracle (even if approximate) over these objects. Finally, using the combinatorial structure of symmetric Nash-equilibria (SNE) when both players play bases of matroids, we show that these can be found with a single projection or convex minimization (without using online learning). 
\end{abstract}

\begin{keywords}
multiplicative weights update, generalized approximate counting oracles, online mirror descent, Bregman projections, submodular functions, lexicographically optimal bases, combinatorial games
\end{keywords}

\maketitle

\section{Introduction}\label{intro}
The motivation of our work comes from two-player zero-sum games where both players play combinatorial objects, such as spanning trees, cuts, matchings, or paths in a given graph. The number of pure strategies of both players can then be exponential in a natural description of the problem. These are {\it succinct} games, as discussed in the paper of  \cite{papadimitriou2008computing} on correlated equilibria. For example, in a spanning tree game in which all the results of this paper apply, pure strategies correspond to spanning trees $T_1$ and $T_2$ selected by the two players in a graph $G$ (or two distinct graphs $G_1$ and $G_2$) and the payoff $\sum_{e\in T_1, f\in T_2} L_{ef}$ is a bilinear function; this allows for example to model classic network interdiction games (see for e.g., \cite{Washburn1995}), design problems (\cite{Chakrabarty2006}), and the interaction between algorithms for many problems such as ranking and compression as \emph{bilinear duels} (\cite{Immorlica2011}). To formalize the games we are considering, assume that the pure strategies for player 1 (resp.~player 2) correspond to the vertices $u$ (resp.~$v$) of a {\it strategy} polytope $P \subseteq \mathbb{R}^m$ (resp.~$Q\subseteq \mathbb{R}^n$) and that the loss for player 1 is given by the bilinear function $u^TLv$ where $L\in \mathbb{R}^{m\times n}$. A feature of bilinear loss functions is that the bilinearity extends to mixed strategies as well, and thus one can easily see that mixed Nash equilibria (see discussion in Section \ref{prelim}) correspond to solving the min-max problem: 
\begin{equation} \label{eqNE}
\min_{x\in P} \max_{y\in Q} x^TLy=\max_{y\in Q}\min_{x\in P} x^TLy.
\end{equation}
We refer to such games as MSP (Min-max Strategy Polytope) games. 

Nash equilibria for two-player zero-sum games can be characterized and found by solving a linear program (\cite{Neumann1928}). However, for succinct games in which the strategies of both players are exponential in a natural description of the game, the corresponding linear program has exponentially many variables and constraints, and as  \cite{papadimitriou2008computing} point out in their open questions section, ``there are no standard techniques for linear programs that have both dimensions exponential.'' Under bilinear losses/payoffs however, the von Neumann linear program can be reformulated in terms of the strategy polytopes $P$ and $Q$, and this reformulation can be solved using the equivalence between optimization and separation and the ellipsoid algorithm (\cite{Grotschel1981}) (discussed in more detail in Section \ref{prelim}). In the case of the spanning tree game mentioned above, the strategy polytope of each player is simply the spanning tree polytope characterized by \cite{Edmonds1971}. Note that \cite{Immorlica2011} give such a reformulation for bilinear games involving strategy polytopes with compact formulations only (i.e., each strategy polytope must be described using a polynomial number of inequalities).

In this paper, we first explore ways of solving efficiently this linear program using learning algorithms. As is well-known, if one of the players uses a no-regret learning algorithm and adapts his/her strategies according to the losses incurred so far (with respect to the most adversarial opponent strategy) then the average of the strategies played by the players in the process constitutes an approximate equilibrium (\cite{Cesa-Bianchi2006Book}). Therefore, we consider the following learning problem over $T$ rounds in the presence of an adversary: in each round the ``learner'' (or player) chooses a mixed strategy $x_t \in P$. Simultaneously, the adversary chooses a loss vector $l_t = Lv_t$ where $v_t \in Q$ and the loss incurred by the player is $x_t^T l_t$. The goal of the player is to minimize the cumulative loss, i.e., $\sum_{i=1}^t x_t^T l_t$. Note that this setting is similar to the classical full-information online structured learning problem (see for e.g., \cite{Audibert2013}), where the learner is required to play a pure strategy $u_t \in \mathcal{U}$ (where $\mathcal{U}$ is the vertex set of $P$) possibly randomized according to a mixed strategy $x_t \in P$, and aims to minimize the loss in expectation, i.e., $\mathbb{E}(u_t^Tl_t)$. 

Our two main results on learning algorithms are (i) an efficient implementation of online mirror descent when $P$ is the base polytope of a polymatroid and this is obtained by designing a novel algorithm to compute Bregman projections over such a polytope, and (ii) an efficient implementation of the MWU algorithm over the vertices of 0/1 polytopes $P$, provided we have access to a generalized counting oracle for the vertices.  These are discussed in detail below. In both cases, we assume that we have an (approximate) linear optimization oracle for $Q$, which allows to compute the (approximately) worst loss vector given a mixed strategy in $P$. Finally, we study the combinatorial properties of symmetric Nash-equilibria for matroid games and show how this structure can be exploited to compute these equilibria. 

\subsection{Online Mirror Descent} 
Even though the online mirror descent algorithm is near-optimal in terms of regret for most of online learning problems (\cite{Srebro2011}), it is not computationally efficient. One of the crucial steps in the mirror descent algorithm is that of taking Bregman projections on the strategy polytope that has closed form solutions for simple cases like the Euclidean ball or the $n$-dimensional simplex, and this is why such polytopes have been the focus of attention. For general polyhedra (or convex sets) however, taking a Bregman projection is a separable convex minimization problem. One could exploit the general machinery of convex optimization such as the ellipsoid algorithm, but the question is if we can do better. 

\paragraph{\textbf{First contribution:}} We give a novel primal-style algorithm, {\sc Inc-Fix} for minimizing separable strongly convex functions over the base polytope $P$ of a polymatroid. This includes the setting in which ${\mathcal U}$ forms the bases of a matroid, and cover many interesting examples like k-sets (uniform matroid), spanning trees (graphic matroid), matchable/reachable sets in graphs (matching matroid/gammoid), etc. The algorithm is iterative and maintains a feasible point in the polymatroid (or the independent set polytope for a matroid). This point follows a trajectory that is guided by two constraints: (i) the point must remain in the polymatroid, (ii) the smallest indices (not constrained by (i)) of the gradient of the Bregman divergence increase uniformly. The correctness of the algorithm follows from first order optimality conditions, and the optimality of the greedy algorithm when optimizing linear functions  over polymatroids. We discuss special cases of our algorithm under two classical mirror maps, the unnormalized entropy and the Euclidean mirror map, under which each iteration of our algorithm reduces to staying feasible while moving along  a straight line.  We believe that the key ideas developed in this algorithm may be useful to compute projections under Bregman divergences over other polyhedra, as long as the linear optimization for those is well-understood. 

As a remark, in order to compute $\epsilon$-approximate Nash-equilibria, if both the strategy polytopes are polymatroids then the same projection algorithms apply to the saddle-point mirror prox algorithm (\cite{Nemirovski2004Prox}) and reduce the dependence of the rate of convergence on $\epsilon$ to $O(1/\epsilon)$. 

\subsection{Multiplicative Weights Update}
The \emph{multiplicative weights update} (MWU) algorithm proceeds by maintaining a probability distribution over all the pure strategies of each player, and multiplicatively updates the probability distribution in response to the adversarial strategy. The number of iterations the MWU algorithm takes to converge to an $\epsilon-$approximate strategy\footnote{A strategy pair $(x^*, y^*)$ is called an $\epsilon$-approximate Nash-equilibrium if $x^{*T}Ly - \epsilon \leq x^{*T}Ly^* \leq x^TLy^* +\epsilon$ for all $x \in P, y \in Q$.} is $O(\ln N/\epsilon^2)$, where $N$ is the number of pure strategies of the learner, in our case the size of the vertex set of the combinatorial polytope. However, the running time of each iteration is $O(N)$ due to the updates required on the probabilities of each pure strategy. A natural question is, \emph{can the MWU algorithm be adapted to run in logarithmic time in the number of strategies?} 

\paragraph{\textbf{Second contribution:}} We give a general framework for simulating the MWU algorithm over the set of vertices ${\mathcal U}$ of a 0/1 polytope $P$ efficiently 
by updating product distributions. A product distribution $p$ over $\mathcal{U} \subseteq \{0,1\}^m$ is such that $p(u) \propto \prod_{e: u(e)=1} \lambda(e)$ for some multiplier vector $\lambda \in \mathbb{R}^m_+$. Note that it is easy to start with a uniform distribution over all vertices in this representation, by simply setting $\lambda(e)=1$ for all $e$. 
 The key idea is that a \emph{multiplicative weight update to a product distribution results in another product distribution}, obtained by appropriately and multiplicatively updating each $\lambda(e)$. Thus, in different rounds of the MWU algorithm, we move from one product distribution to another. This implies that we can restrict our attention to product distributions without loss of generality, and this was already known as any point in a 0/1 polytope can be decomposed into a product distribution. Product distributions allow us to maintain a distribution on (the exponentially sized) ${\mathcal U}$ by simply maintaining $\lambda \in \mathbb{R}^m_+$.
To be able to use the MWU algorithm together with product distributions, we require access to a {\it generalized (approximate) counting oracle} which, given $\lambda\in \mathbb{R}^m_+$, (approximately) computes $\sum_{u\in {\mathcal U}}\prod_{e: u_e=1} \lambda(e)$ and also, for any element $f$, computes $\sum_{u\in {\mathcal U}: u_f=1}\prod_{e: u_e=1} \lambda(e)$ allowing the derivation of the corresponding marginals $x \in P$. For self-reducible structures ${\mathcal U}$ (\cite{Schnorr76}) (such as spanning trees, matchings or Hamiltonian cycles), the latter condition for every element $f$ is superfluous, and the generalized approximate counting oracle can be replaced by a fully polynomial approximate generator as shown by \cite{JerrumVV86}. 

Whenever we have access to a generalized approximate counting oracle, the MWU algorithm converges to $\epsilon$-approximate in $O(\ln |\mathcal{U}|/\epsilon^2)$ time. A generalized exact counting oracle is available for spanning trees (this is Kirchhoff's determinantal formula or matrix tree theorem) or more generally for bases of regular matroids (see for e.g., \cite{Welsh08}) and randomized approximate ones for bipartite matchings (\cite{Jerrum2004}) and extensions such as $0-1$ circulations in directed graphs or subgraphs with prespecified degree sequences (\cite{Jerrum2004}).
 

As a remark, if a generalized approximate counting oracle  exists for both strategy polytopes (as is the case for the spanning tree game mentioned early in the introduction), the same ideas apply to the optimistic mirror descent algorithm (\cite{Rakhlin2013}) and reduce the dependence of the rate of convergence on $\epsilon$ to $O(1/\epsilon)$ while maintaining polynomial running time.

\subsection{Structure of symmetric Nash equilibria}
For matroid MSP games, we study properties of Nash equilibria using combinatorial arguments with the hope to find constructive algorithms to compute Nash equilibria. Although we were not able to solve the problem for the general case, we prove the following results for the special case of symmetric Nash equilibria (SNE) where both the players play the same mixed strategy. 

\paragraph{\textbf{Third contribution:}} We combinatorially characterize the structure of symmetric Nash equilibria. We prove uniqueness of SNE under positive and negative definite loss matrices. We show that SNE coincide with lexicographically optimal points in the base polytope of the matroid in certain settings, and hence show that they can be efficiently computed using any separable convex minimization algorithm (for example, algorithm {\sc Inc-Fix} in Section \ref{mirror}). Given an \emph{observed} Nash-equilibrium, we can also construct a possible loss matrix for which that is a symmetric Nash equilibrium.\\

\paragraph{\textbf{Comparison of approaches.}} Both the learning approaches have different applicability and limitations. We know how to efficiently perform the Bregman projection only for polymatroids, and not for bipartite matchings for which the MWU algorithm with product distributions can be used. On the other hand, there exist matroids (\cite{Azar94}) for which any generalized approximate counting algorithm requires an exponential number of calls to an independence oracle, while an independence oracle is all what we need to make the Bregman projection efficient in the online mirror descent approach. Our characterization of symmetric Nash-equilibria shows that a single projection is enough to compute symmetric equilibria (and check if they exist).

This paper is structured as follows. After discussing related work in Section \ref{related}, we review background and notation in Section \ref{prelim} and show a reformulation for the game that can be solved using separation oracles for the strategy polytopes. We give a primal-style algorithm for separable convex minimization over base polytopes of polymatroids in Section \ref{mirror}. We discuss the special cases of computing Bregman projections under the entropy and Euclidean mirror maps, and show how to solve the subproblems that arise in the algorithm. In Section \ref{mwu}, we show how to simulate the multiplicative weights algorithm in polynomial time using generalized counting oracles. We further show that the MWU algorithm is robust to errors in these algorithms for optimization and counting. Finally in Section \ref{sne}, we completely characterize the structure of symmetric Nash-equilibria for matroid games under symmetric loss matrices. Sections \ref{mirror}, \ref{mwu} and \ref{sne} are independent of each other, and can be read in any order.

\section{Related work}\label{related}
The general problem of finding Nash-equilibria in 2-player games is PPAD-complete (\cite{Chen2006}, \cite{Daskalakis2009}). Restricting to two-player zero-sum games without any assumptions on the structure of the loss functions, asymptotic upper and lower bounds (of the order $O(\log N/\epsilon^2)$) on the support of $\epsilon$-approximate Nash equilibria with $N$ pure strategies are known (\cite{Althofer1994}, \cite{Lipton1994}, \cite{Lipton2003}, \cite{Feder2007}). These results perform a search on the support of the Nash equilibria and it is not known how to find Nash equilibria for large two-player zero sum games in time logarithmic in the number of pure strategies of the players. In recent work, \cite{Hazan2015} show that any online algorithm requires $\tilde{\Omega}(\sqrt{N})$ time to approximate the value of a two-player zero-sum game, even when given access to constant time best-response oracles. In this work, we restrict our attention to two-player zero-sum games with \emph{bilinear} loss functions and give polynomial time algorithms for finding Nash-equilibria (polynomial in the representation of the game). 


One way to find Nash-equilibria is by using regret-minimization algorithms. We look at the problem of learning combinatorial concepts that has recently gained a lot of popularity in the community (\cite{Koolen2015}, \cite{Audibert2013}, \cite{Neu2015}, \cite{Cohen2015}, \cite{Ailon2014}, \cite{Koolen2010}), and has found many real world applications in communication, principal component analysis, scheduling, routing, personalized content recommendation, online ad display etc. There are two popular approaches to learn over combinatorial concepts. The first is the Follow-the-Perturbed-Leader (\cite{Kalai2005}), though efficient in runtime complexity it is suboptimal in terms of regret (\cite{Neu2015}, \cite{Cohen2015}). The other method for learning over combinatorial structures is the online mirror descent algorithm (online convex optimization is attributed to \cite{Zinkevich2003} and mirror descent to \cite{Nemirovski1983}), which fares better in terms of regret but is computationally inefficient (\cite{Srebro2011}, \cite{Audibert2013}). In the first part of the work, we give a novel algorithm that speeds up the mirror descent algorithm in some setting. 

\cite{Koolen2010} introduce the Component Hedge (CH) algorithm for linear loss functions that sum over the losses for each component. They perform multiplicative updates for each component and subsequently perform Bregman projections over certain extended formulations with a polynomial number of constraints, using off-the-shelf convex optimization subroutines. Our work, on the other hand, allows for polytopes with exponentially many inequalities. 
The works of \cite{Helmbold2009} (for learning over permutations) and \cite{Warmuth2008randomized} (for learning $k$-sets) have been shown to be special cases of the CH algorithm. Our work applies to these settings. 

One of our contributions is to give a new efficient algorithm {\sc Inc-Fix} for performing Bregman projections on base polytopes of  polymatroids, a step in the online mirror descent algorithm. Another way of solving the associated separable convex optimization problem over base polytopes of polymatroids is the decomposition algorithm of \cite{Groenevelt1991} (also studied in \cite{Fujishige2005}). Our algorithm is fundamentally different from the decomposition algorithm; the latter generates a sequence of violated inequalities (a dual approach) while our algorithm maintains a feasible point in the polymatroid (a primal approach). The violated inequalities in the decomposition algorithm come from a series of submodular function minimizations, and their computations can be amortized into a parametric submodular function minimization as is shown in \cite{Nagano2007,Nagano2007b,Suehiro2012}. In the special cases of the Euclidean or the entropy mirror maps, our iterative algorithm leads to problems involving the maintenance of feasibility along lines, which reduces to parametric submodular function minimization problems.  In the special case of cardinality-based submodular functions ($f(S)=g(|S|)$ for some concave $g$), our algorithm can be implemented overall in $O(n^2)$ time, matching the running time of a specialized algorithm due to \cite{Suehiro2012}.  


We also consider the multiplicative weights update (MWU) algorithm (\cite{Littlestone1994}, \cite{Freund1999}, \cite{Arora2012}) rediscovered for different settings in game theory, machine learning, and online decision making with a large number of applications. Most of the applications of the MWU algorithm have running times polynomial in the number of pure strategies of the learner, an observation also made in \cite{Blum2008}. In order to perform this algorithm efficiently for exponential experts (with combinatorial structure), it does not take much to see that multiplicative updates for linear losses can be made using product terms. However, the analysis of prior works was very specific to the structure of the problem. For example, \cite{Takimoto2003} give efficient implementations of the MWU for learning over general $s-t$ paths that allow for cycles or over simple paths in acyclic directed graphs. This approach relies on the recursive structure of these paths, and does not generalize to simple paths in an undirected graph (or a directed graph). Similarly, \cite{Helmbold1997} rely on the recursive structure of bounded depth binary decision trees. \cite{Koo2007} use the matrix tree theorem to learn over spanning trees by doing large-margin optimization. We give a general framework to analyze these problems, while drawing a connection to sampling or generalized counting of product distributions.  

\section{Preliminaries}\label{prelim}

In a two-player zero-sum game with loss (or payoff) matrix $R\in \mathbb{R}^{M\times N}$, a mixed strategy $x$ (resp.~$y$) for the row player (resp.~column player) trying to minimize (resp.~maximize) his/her loss is an assignment $x\in \Delta_M$ (resp.~$y\in \Delta_N$) where $\Delta_K$ is the simplex $\{x \in \mathbb{R}^{K}, \sum_{i=1}^K x_i = 1, x\geq 0\}$. A pair of mixed strategies $(x^*,y^*)$ is called a Nash-equilibrium if $x^{*T} R \bar{y} \leq x^{*T}Ry^*\leq \bar{x}^T R y^*$ for all $\bar{x} \in \Delta_M, 
\bar{y} \in \Delta_N$, i.e. there is no incentive for either player to switch from $(x^*, y^*)$ given that the other player does not deviate. Similarly, a pair of strategies $(x^*,y^*)$ is called an $\epsilon-$approximate Nash-equilibrium if $x^{*T}R\bar{y} - \epsilon \leq x^{*T}R y^* \leq \bar{x}^TR y^*+\epsilon$ for all $\bar{x} \in \Delta_M, \bar{y} \in \Delta_N$. Von Neumann showed that every two-player zero-sum game has a mixed Nash-equilibrium that can be found by solving the following dual pair of linear programs: 
\begin{align*} 
(LP1): & \min \lambda & (LP2): & \max \mu\\
& R^T x \leq \lambda e, & & Ry \geq \mu e, \\
& e^{T}x = 1, x\geq 0. & & e^{T}y =1, y\geq 0.
\end{align*}
where $e$ is a vector of all ones in the appropriate dimension. 

In our two-player zero-sum MSP games, we let the strategies of the row player be $\mathcal{U}$ = vert($P$), where $P = \{x \in \mathbb{R}^m, Ax \leq b\}$ is a polytope and vert($P$) is the set of vertices of $P$ and those of the \emph{column} player be $\mathcal{V}$ = vert($Q$) where $Q = \{y \in \mathbb{R}^n, Cy \leq d\}$ is also a polytope. The numbers of pure strategies, $M=|\mathcal{U}|$, $N=|\mathcal{V}|$ will typically be exponential in $m$ or $n$, and so may be the number of rows in the  constraint matrices A and C. The linear programs $(LP1)$ and $(LP2)$ have thus exponentially many variables and constraints. We restrict our attention to bilinear loss functions  that are represented as $R_{uv} = u^{T}Lv$ for some $m\times n$ matrix $L$. 
An artifact of bilinear loss functions is that the bilinearity extends to mixed strategies as well. If $\lambda\in \Delta_{\mathcal{U}}$ and $\theta\in \Delta_{\mathcal{V}}$ are mixed strategies for the players then the expected loss is equal to $x^TLy$ where $x = \sum_{u \in \mathcal{U}} \lambda_u u$ and $y = \sum_{v \in \mathcal{V}} \theta_v v$: 
\begin{align*}
{\mathbb E}_{u,v}(R_{uv}) &= \sum_{u \in \mathcal{U}}\sum_{v \in \mathcal{V}} \lambda_u \theta_v (u^{T}Lv ) = (\sum_{u \in \mathcal{U}} \lambda_u u) L (\sum_{v \in \mathcal{V}} \theta_v v) = x^{T}Ly.
\end{align*}
Thus the loss incurred by mixed strategies only depend on the {\it marginals} of the distributions over the vertices of $P$ and $Q$; distributions with the same marginals give the same expected loss. This plays a crucial role in our proofs. Thus the Nash equilibrium problem for MSP games reduces to (\ref{eqNE}): $\min_{x\in P} \max_{y\in Q} x^TLy=\max_{y\in Q}\min_{x\in P} x^TLy.$

As an example of an MSP game, consider a \emph{spanning tree} game where the pure strategies of each player are the spanning trees of a given graph $G = (V,E)$ with $m$ edges, and $L$ is the $m\times m$ identity matrix. This corresponds to the game in which 
the row player would try to minimize the intersection of his/her spanning tree with that of the column player, whereas the column player would try to maximize the intersection. For a complete graph, the number of pure strategies for each player is $n^{n-2}$ by Cayley's theorem, where $n$ is the number of vertices. For the graph $G$ in Figure \ref{fig:example1}(a), the marginals of the unique Nash equilibrium for both players are given in \ref{fig:example1}(b) and (c). For graphs whose blocks are uniformly dense, both players can play the same (called symmetric) optimal mixed strategy.   
\begin{figure}[htbp]
\begin{center}
\includegraphics[width=\textwidth]{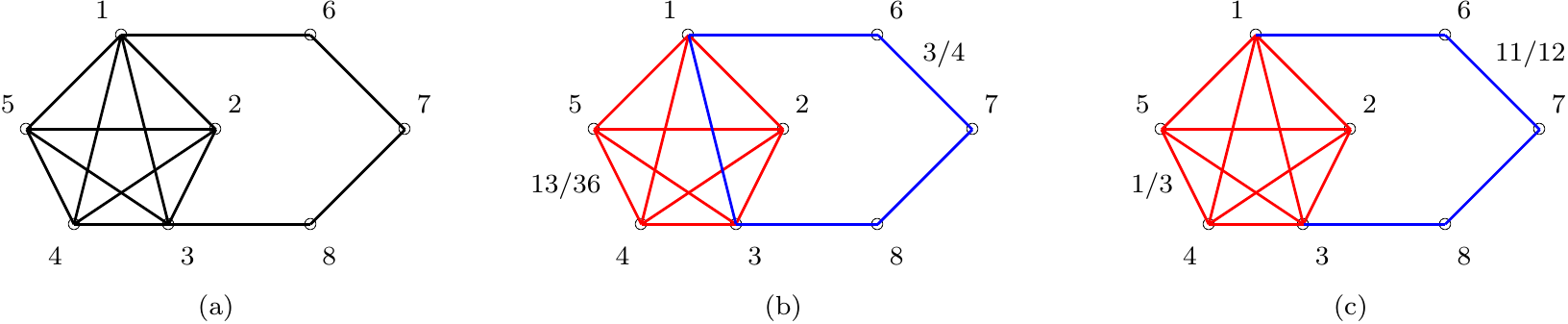}
\caption{\footnotesize (a) $G = (V,E)$, (b) Optimal strategy for the row player minimizing the number of edges in the intersection of the two trees, (c) Optimal strategy for the column player maximizing the number of edges in the intersection.}
\label{fig:example1}
\end{center}
\end{figure}

For MSP games with bilinear losses, the linear programs $(LP1)$ and $(LP2)$ can be reformulated over the space of marginals, and $(LP1)$ becomes
\begin{align} 
(LP1'):  &\min \lambda \nonumber\\
& x^{T}Lv \leq \lambda \;\;\forall\; v \in \mathcal{V}, \label{c_f1}\\
& x \in P\subseteq {\mathbb R}^m, \label{c_f2}
\end{align}
and similarly for $(LP2)$: $\max \{\mu: u^TLy \geq \mu \;\; \forall u\in \mathcal{U}, y\in Q\}$.  This reformulation can be used to show that, for these MSP games with bilinear losses (and exponentially many strategies), there exists a Nash equilibrium with small (polynomial) encoding length. A polyhedron $K$ is said to have \emph{vertex-complexity} at most $\nu$ if there exist finite sets $V, E$ of rational vectors such that $K = \textrm{conv}(V) + \textrm{cone}(E)$ and such that each of the vectors in $V$ and $E$ has encoding length at most $\nu$. A polyhedron $K$ is said to have \emph{facet-complexity} at most $\phi$ if there exists a system of inequalities with rational coefficients that has solution set $K$ such that the (binary) encoding length of each inequality of the system is at most $\phi$.
Let $\nu_P$ and $\nu_Q$ be the vertex complexities of polytopes $P$ and $Q$ respectively; if $P$ and $Q$ are $0/1$ polytopes, we have $\nu_P\leq m$ and $\nu_Q \leq n$. This means that the facet complexity of $P$ and $Q$ are $O(m^2 \nu_P)$ and $O(n^2\nu_Q)$ (see Lemma (6.2.4) in \cite{lovasz1988geometric}). Therefore the facet complexity of the polyhedron in $(LP1')$ can be seen to be  $O(\max(m \langle L\rangle \nu_Q,  n^2\nu_P))$, where $\langle L\rangle$ is the binary enconding length of $L$ and the first term in the $\max$ corresponds to the inequalities (\ref{c_f1}) and  the second to (\ref{c_f2}). From this, we can derive Lemma \ref{vertexcompl}.

\begin{lemma} \label{vertexcompl} 
The vertex complexity of the linear program $(LP1')$ is $O(m^2(m\langle L \rangle \nu_Q + n^2 \nu_P))$ where  $\nu_P$ and $\nu_Q$ are  the vertex complexities of  $P$ and $Q$ and $\langle L\rangle$ is the binary encoding length of $L$. (If $P$ and $Q$ are $0/1$ polytopes then $\nu_P\leq m$ and $\nu_Q \leq n$.)
\end{lemma} 

This means that our polytope defining $(LP1')$ is well-described (\`a la Gr\"{o}tschel et al.). We can thus use the machinery of the ellipsoid algorithm (\cite{Grotschel1981}) to find a Nash Equilibrium in polynomial time for these MSP games, provided we can optimize (or separate) over $P$ and $Q$.  Indeed, by the ellipsoid algorithm, we have the equivalence between strong separation and strong optimization for well-described polyhedra. The strong separation over (\ref{c_f1}) reduces to strong optimization over $Q$, while a strong separation algorithm over (\ref{c_f2}), i.e. over $P$, can be obtained from a strong separation over $P$ by the ellipsoid algorithm.

We should also point out at this point that, if the polyhedra $P$ and $Q$ admit a compact extended formulation then $(LP1')$ can also be reformulated in a compact way (and solved using interior point methods, for example). A compact extended formulation for a polyhedron $P\subseteq {\mathbb R}^d$ is a polytope with polynomially many (in $d$) facets in a higher dimensional space that projects onto $P$. This allows to give a compact extended formulation for $(LP1')$ for the spanning tree game as a  compact formulation is known for the spanning tree polytope (\cite{Martin1991}) (and any other game where the two strategy polytopes can be described using polynomial number of inequalities). However, this would not work for a corresponding matching game since the extension complexity for the matching polytope is exponential (\cite{Rothvoss2014}).

\section{Online mirror descent}\label{mirror}
In this section, we show how to perform mirror descent faster, by providing algorithms for minimizing strongly convex functions over base polytopes of polymatroids. 

Consider a compact convex set $X \subseteq \mathbb{R}^n$, and let $\mathcal{D} \subseteq \mathbb{R}^n$ be a convex open set such that $X$ is included in its closure. A mirror map (or a distance generating function) is a $k$-strongly convex function\footnote{A function $f$ is said to be $k$-strongly convex over domain $\mathcal{D}$ with respect to a norm $||\cdot||$ if $f(x) \geq f(y) + \nabla f(y)^T(x-y) + \frac{k}{2}||x-y||^2$.} and differentiable function $\omega : \mathcal{D} \rightarrow \mathbb{R}$ that satisfies additional properties of divergence of the gradient on the boundary of $\mathcal{D}$, i.e., $\lim_{x \rightarrow \partial \mathcal{D}} || \nabla \omega (x)|| = \infty$ (for details, refer to \cite{Nemirovski1983}, \cite{Beck2003}, \cite{Bubeck2014}). In particular, we consider two important mirror maps in this work, the Euclidean mirror map and the unnormalized entropy mirror map. The Euclidean mirror map is given by $\omega(x) = \frac{1}{2} ||x||^2$, for $\mathcal{D} = \mathbb{R}^E$ and is 1-strongly convex with respect to the $L_2$ norm. The unnormalized entropy map is given by $\omega(x) = \sum_{e \in E} x(e) \ln (x(e)) - \sum_{e \in E} x(e)$, for $\mathcal{D} = \mathbb{R}_+^{E}$ and is 1-strongly convex over the unnormalized simplex with respect to the $L_1$ norm (see proof in Appendix \ref{app:mirror_descent}, Lemma \ref{strong-convexity}). 

The Online Mirror Descent (OMD) algorithm with respect to the mirror map $\omega$ proceeds as follows. Think of $X$ as the strategy polytope of the row player or the learner in the online algorithm that is trying to minimize regret over the points in the polytope with respect to loss vectors $l^{(t)}$ revealed by the adversary in each round. The iterative algorithm starts with the first iterate $x^{(1)}$ equal to the $\omega$-center of $X$ given by $\arg\min_{x\in X} \omega(x)$. Subsequently, for $t>1$, the algorithm first moves in unconstrained way using 
$$ \nabla\omega(y^{(t+1)}) = \nabla\omega(x^{(t)}) - \eta \nabla l^{(t)}(x^{(t)}).$$
Then the next iterate $x^{(t+1)}$ is obtained by a projection step:
\begin{equation}  \label{gradientstep}
x^{(t+1)} = \arg\min_{x \in X \cap \mathcal{D}} D_{\omega} (x, y^{(t+1)}),
\end{equation}
where the generalized notion of projection is defined by  the function $D_{\omega} (x, y) = \omega(x) - \omega(y) - \nabla \omega (y)^{T}(x-y)$, called the Bregman divergence of the mirror map. Since the Bregman divergence is a strongly convex function in $x$ for a fixed $y$, there is a unique minimizer $x^{(t+1)}$. For the two mirror maps discussed above, the divergence is $D_{\omega}(x,y) = \frac{1}{2} ||x-y||^2$ for $y \in R^{E}$ for the Euclidean mirror map, and is $D_{\omega}(x,y) = \sum_{e \in E} x(e) \ln (x(e)/y(e)) - \sum_{e \in E} x(e) + \sum_{e\in E} y(e)$ for the entropy mirror map.

The regret of the online mirror descent algorithm is known to scale as $O(RG\sqrt{t})$ where $R$ depends on the geometry of the convex set and is given by $R^2 = \arg\max_{x \in X} \omega(x) - \arg\min_{x \in X} \omega(x)$ and $G$ is the Lipschitz constant of the underlying loss functions, i.e., $||\nabla l^{(i)}||_* \leq G$ for all $i = 1, \hdots, T$. We restate the theorem about the regret of the online mirror-descent algorithm (adapted from \cite{Bubeck2011}, \cite{Ben2011}, \cite{Rakhlin2014}). 

\begin{theorem} \label{bubeck_omd} Consider online mirror descent based on a k-strongly convex (with respect to $||\cdot ||$) and differentiable mirror map $\omega: \mathcal{D} \rightarrow \mathbb{R}$ on a closed convex set X. Let each loss function $l^{(i)}: X\rightarrow \mathbb{R}$ be convex and G-Lipschitz, i.e. $||\nabla l^{(i)} ||_* \leq G \;\;\forall i \in \{1, \hdots, t\}$ and let the radius $R^2 = \arg\max_{x\in X} \omega(x) - \arg\min_{x \in X} \omega(x)$. Further, we set $\eta = \frac{R}{G} \sqrt{\frac{2k}{t}}$ then:
$$ \sum_{i=1}^t l^{(i)}(x_i) - \sum_{i=1}^t l^{(i)}(x^*) \leq RG\sqrt{\frac{2t}{k}} \;\;\; \textrm{ for all } x^* \in X.$$
\end{theorem} 

\subsection{Convex Minimization on Base Polytopes of Polymatroids}\label{convex_min}
In this section, we consider the setting in which $X$ is the base polytope of a polymatroid. Let $f$ be a monotone submodular function, i.e., $f$ must satisfy the following conditions: (i) (monotonicity) $f(A) \leq f(B)$ for all $A \subseteq B \subseteq E$, and (ii) (submodularity) $f(A) + f(B) \geq f(A \cup B) + f(A \cap B)$ for all $A, B \subseteq E$. Furthermore, we will assume $f(\emptyset) =0$ (normalized) and without loss of generality we assume that $f(A)>0$ for $A \neq \emptyset$. Given such a function $f$, the independent set polytope is defined as $P(f) = \{x \in \mathbb{R}_+^\mathrm{E} : x(U) \leq f(U)\;\forall\; U \subseteq E\}$ and the base polytope as $B(f) = \{x \in \mathbb{R}_+^\mathrm{E} : x(E) = f(E),  x(U) \leq f(U)\;\forall\;U \subseteq E\}$ (\cite{Edmonds1970}). A typical example is when $f$ is the rank function of a matroid, and the corresponding base polytope corresponds to the convex hull of its bases. Bases of a matroid include spanning trees (bases of a graphic matroid), $k$-sets (uniform matroid), maximally matchable sets of vertices in a graph (matching matroid), or maximal subsets of $T\subseteq V$ having disjoint paths from vertices in $S\subseteq V$ in a directed graph $G=(V,E)$ (gammoid).  

Let us consider any strongly convex separable function $h: \mathcal{D} \rightarrow \mathbb{R}$, defined over a convex open set $\mathcal{D}$ such that $P(f) \subseteq \overline{\mathcal{D}}$ (i.e., closure of $\mathcal{D}$) and $\nabla h(\mathcal{D}) = \mathbb{R}^E$. We require that either $0 \in \mathcal{D}$ or there exists some $x \in P(f)$ such that $\nabla h(x) = c\chi(E), c \in \mathbb{R}$. These conditions hold, for example, for minimizing the Bregman divergence of the two mirror maps (Euclidean and entropy) we discussed in the previous section. We present in this section an algorithm {\sc Inc-Fix}, that minimizes such convex functions $h$ over the base polytope $B(f)$ of a given monotone normalized submodular function $f$. Our approach can be interpreted as a generalization of Fujishige's \emph{monotone algorithm} for finding a lexicographically optimal base to handle general separable convex functions.

\paragraph{\textbf{Key idea:}} The algorithm is iterative and maintains a vector $x \in P(f) \cap \mathcal{D}$. When considering $x$ we associate a weight vector given by $\nabla h(x)$ and consider the set of minimum weight elements. We move $x$ within $P(f)$ in a direction such that $(\nabla h(x))_e$ increases uniformly on the minimum weight elements, until one of two things happen: (i) either continuing further would violate a constraint defining $P(f)$, or (ii) the set of elements of minimum weight changes. If the former happens, we \emph{fix} the tight elements and continue the process on non-fixed elements. If the latter happens, then we continue increasing the value of the elements in the modified set of minimum weight elements. The complete description of the {\sc Inc-Fix} algorithm is given in Algorithm \ref{inc-fix}. We refer to the initial starting point as $x^{(0)}$. The algorithm constructs a sequence of points $x^{(0)}, x^{(1)},$ $\hdots, x^{(k)} = x^*$ in $P(f)$. At the beginning of iteration $i$, the set of \emph{non-fixed} elements whose value can potentially be increased without violating any constraint is referred to as $N_{i-1}$. The iterate $x^{(i)}$ is obtained by \emph{increasing} the value of minimum weight elements of $x^{(i-1)}$ in $N_{i-1}$ weighted by $(\nabla h(x))_e$ such that the resulting point stays in $P(f)$. Iteration $i$ of the main loop ends when some non-fixed element becomes tight and we \emph{fix} the value on these elements by updating $N_i$. We continue until all the elements are fixed, i.e., $N_i = \emptyset$. We denote by $T(x): \mathbb{R}^\mathrm{E} \rightarrow 2^\mathrm{E}$ the maximal set of tight elements in $x$ (which is unique by submodularity of $f$). 

\begin{algorithm}
\caption{{\sc Inc-Fix}}
\label{inc-fix}
 \Input{$f: 2^E \rightarrow \mathbb{R}$, $h = \sum_{e \in E} h_e$, and input $x^{(0)}$}
 \Output{$x^* = \arg\min_{z\in B(f)} \sum_{e} h_e(z(e))$}
{$N_0 = E, i=0$\;}\\
 \Repeat(\tcc*[f]{Main loop}){$N_{i} = \emptyset$}{
 	$i \leftarrow i+1$\;\\
	$x = x^{(i-1)}$\;\\
	$M = \arg \min_{e \in N_{i-1}} \nabla (h(x))_e$\;\\
	\While(\tcc*[f]{Inner loop}){$T(x) \cap M = \emptyset$} {
		$\epsilon_1 = \max \{ \delta : (\nabla h)^{-1}(\nabla h (x) + \delta \chi(M)) \in P(f)\}$\;\\
		$\epsilon_2 = \arg\min_{e\in N_{i-1} \setminus M} (\nabla h(x))_e - \arg\min_{e \in N_{i-1}} (\nabla h(x))_e$\;\\		         $x \leftarrow (\nabla h)^{-1}(\nabla h(x) + \min(\epsilon_1, \epsilon_2) \chi(M));$\tcc*[f]{Increase}\\
		$M = \arg \min_{e \in N_{i-1}} (\nabla h(x))_e$\;
	}
	$x^{(i)} = x$, $M_i = \arg \min_{e \in N_{i-1}} (\nabla h(x^{(i)}))_e$\;
	$N_{i} = N_{i-1} \setminus (M_i \cap T(x^{(i)}))$\tcc*[f]{Fix}
}
Return $x^{*} = x^{(i)}$.
\end{algorithm}
\paragraph{\textbf{Choice of starting point:}} We let $x^{(0)}=0$ unless $0 \notin \mathcal{D}$; observe that $0 \in P(f) \subseteq \overline{\mathcal{D}}$. In the latter case, we let $x^{(0)} \in P(f)$ such that $\nabla h(x^{(0)}) = c\chi(E)$ for some $c \in \mathbb{R}$. For example, for the Euclidean mirror map and some $y \in \mathbb{R}^{E}$, $\nabla h(x) = \nabla D_{\omega}(x, y) = x-y$ and $\mathcal{D} = \mathbb{R}^{E}$, hence we start the algorithm with $x^{(0)}=0$. However, for the entropy mirror map and some positive (component-wise) $y \in \mathbb{R}^{E}$, $\nabla h(x) = \nabla D_{\omega}(x, y) = \ln (\frac{x}{y})$ and $\mathcal{D} = \mathbb{R}_{>0}^{E}$. We thus start the algorithm with $x^{(0)} = cy$ for a small enough $c>0$ such that $x \in P(f)$. 

We now state the main theorem to prove correctness of the algorithm. 
\begin{theorem} \label{inc-fix-correctness} 
Consider a $k$-strongly convex and separable function $\sum_{e \in E} h_e(\cdot) : \mathcal{D} \rightarrow \mathbb{R}$ where $\mathcal{D}$ is a convex open set in $\mathbb{R}^E$, $P(f) \subseteq \overline{\mathcal{D}}$ and $\nabla h (\mathcal{D}) = \mathbb{R}^E$ such that either $0 \in \mathcal{D}$ or there exists $x \in P(f)$ such that $\nabla h(x) = c \chi(E)$ for $c \in \mathbb{R}$. Then, the output of {\sc Inc-Fix} algorithm is $x^* = \arg \min_{z \in B(f)} \sum_{e} h_e(z(e))$.
\end{theorem}

The proof relies on the following optimality conditions, which follows from first order optimality conditions and Edmonds' greedy algorithm. 

\begin{theorem} \label{main-optimality} Consider any strongly convex separable function $h(x): \mathcal{D}\rightarrow \mathbb{R}$ where $h(x) = \sum_{e \in E} h_e (x(e))$, and any monotone submodular function $f: 2^{E} \rightarrow \mathbb{R}$ with $f(\emptyset)=0$. Assume $P(f) \subseteq \overline{\mathcal{D}}$ and $\nabla h (\mathcal{D}) = \mathbb{R}^E$. Consider $x^* \in \mathbb{R}^{E}$. Let $F_1, F_2, \hdots, F_k$ be a partition of the ground set $E$ such that $(\nabla h(x^*))_e = c_i$ for all $e \in F_i$ and $c_i < c_j$ for $i<j$. Then, $x^* = \arg\min_{z \in B(f)} \sum_{e \in E} h_e(z(e))$ if and only if $x^*$ lies in the face $H_{opt}$ of $B(f)$ given by
$$H_{opt}:= \{ z \in B(f) | \;z(F_1 \cup \hdots \cup F_i) = f(F_1 \cup \hdots \cup F_i) \; \forall \; 1 \leq i \leq k \}.$$
\end{theorem}

\begin{proof} By first order optimality conditions, we know that $x^* = \arg\min_{z \in B(f)} \sum_{e} h_e(x(e))$ if and only if $\nabla h(x^*)^T(z - x^*) \geq 0$ for all $z \in B(f)$. This is equivalent to $x^* \in \arg \min_{z \in B(f)} \nabla h(x^*)^Tz$. Now consider the partition $F_1, F_2, \hdots, F_k$ as defined in the statement of the theorem. Using Edmonds' greedy algorithm \cite{Edmonds1971}, we know that any $z^* \in B(f)$ is a minimizer of $\nabla h(x^*)^Tz$ if and only if it is tight (i.e., full rank) on each $F_1 \cup \hdots F_{i}$ for $i = 1, \hdots, k$, i.e., $z^*$ lies in the face $H_{opt}$ of $B(f)$ given by
$$H_{opt}:= \{ z \in B(f) | \;z(F_1 \cup \hdots \cup F_i) = f(F_1 \cup \hdots \cup F_i) \; \forall \; 1 \leq i \leq k \}.$$
\end{proof} 

Note that at the end of the algorithm, there may be some elements at zero value (specifically in cases where $x^{(0)}=0$).  In our proof for correctness for the algorithm, we use the following simple lemma about zero-valued elements.   

\begin{lemma} \label{nulledges} For $x \in P(f)$, if a subset $S$ of elements is tight then so is $S \setminus \{e: x(e) = 0\}$.
\end{lemma}
\begin{proof} Let $S = S_1 \cup S_2$ such that $x(S_2) = 0$ and $x(e) > 0 $ for all $e \in S_1$. Then, $f(S_1) \geq x(S_1) = x(S_1 \cup S_2) = f(S_1 \cup S_2) \geq f(S_1)$, where the last inequality follows from monotonicity of $f$, implying that we have equality throughout. Thus, $x(S_1) = f(S_1)$.
\end{proof}

We now give the proof for Theorem \ref{inc-fix-correctness} to show the correctness of the {\sc Inc-Fix} algorithm. 
\begin{proof} Note that since $h(x) = \sum_{e} h_e (x(e))$ is separable and $k$-strongly convex, $\nabla h$ is a strictly increasing function (for each component $e$, $h_e^\prime(x) - h_e^\prime(y)\geq k(x-y)$ for $x>y$). Moreover, $(\nabla h)^{-1}$ is well-defined for all points in $\mathbb{R}^E$ since $\nabla h (\mathcal{D}) = \mathbb{R}^E$. Consider the output of the algorithm $x^*$ and let us partition the elements of the ground set $E$ into $F_1, F_2, \hdots, F_k$ such that $h_e^\prime(x^*(e))  = c_i$ for all $e \in F_i$ and $c_i < c_j$ for $i<j$. We will show that $F_i = M_i \cap T(x^{(i)})$ and that $k$ is the number of iterations of Algorithm \ref{inc-fix}. We first claim that in each iteration $i\geq 1$ of the main loop, the inner loop satisfies the following invariant, as long as the initial starting point $x^{(0)} \in P(f)$:
\begin{enumerate}
\item[(a).] The inner loop returns $x^{(i)}\in P(f)$ such that $h_e^\prime (x^{(i)}(e)) = \max \{h_e^\prime(x^{(i-1)}(e)), \epsilon^{(i)}\}$ for $e \in N_{i-1}$ and some $\epsilon^{(i)} \in \mathbb{R}$, $x^{(i)}(e) = x^{(i-1)}(e)$ for $e \in E \setminus N_{i-1}$, and $T(x^{(i)}) \cap M_i \neq \emptyset$. 
\end{enumerate}
Note that $M$ is initialized to be the set of minimum elements in $N_{i-1}$ with respect to $\nabla h(x^{(i-1)})$ before entering the inner loop. $\epsilon_1$ ensures that the potential increase in $\nabla h (x)$ on elements in $M$ is such that the corresponding point $x \in P(f)$ ($\epsilon_{1}$ exists since $x^{(i-1)} \in P(f)$). $\epsilon_2$ ensures that the potential increase in $\nabla h(x)$ on elements in $M$ is such that $M$ remains the set of minimum weighted elements in $N_{i-1}$. Finally, $x$ is obtained by increasing $\nabla h(x)$ by $\min(\epsilon_1, \epsilon_2)$ and $M$ is updated accordingly. This ensures that at any point in the inner loop, $M = \arg\min_{e \in N_{i-1}} h_e^{\prime}(x(e))$. This continues till there is a tight set $T(x)$ of the current iterate, $x$, that intersects with the minimum weighted elements $M$. Observe that in each iteration of the inner loop either the size of $T(x)$ increases (in the case when $\epsilon_1= \min (\epsilon_1, \epsilon_2)$) or the size of $M$ increases (in the case when $\epsilon_2 = \min(\epsilon_1, \epsilon_2)$). Therefore, the inner loop must terminate. Note that $\epsilon^{(i)} = \min_{e \in N_{i-1}} h_e^\prime(x^{(i)}(e)) = h_f^\prime(x^{(i)}(f))$ for $f \in M_i$, by definition of $M_i$.

Recall that $M_i = \arg \min_{e \in N_{i-1}} (\nabla h(x^{(i)}))_e$ and let the set of elements fixed at the end of each iteration be $L_i = M_i \cap T(x^{(i)})$. We next prove the following claims at the end of each iteration $i \geq 1$. 

\begin{enumerate}
\item[(b).] $x^{(i-1)}(e) \leq x^{(i)}(e)$ for all $e \in E$, as $\nabla h$ is a strictly increasing function and $\nabla h (x^{(i-1)}) \leq \nabla h(x^{(i)})$ due to claim (a).
\item[(c).] Next, observe that we always decrease the set of non-fixed elements $N_i$, i.e., $N_i \subset N_{i-1}$. This follows since $N_i = N_{i-1} \setminus L_i$ and $\emptyset \neq L_i \subseteq N_{i-1}$ (follows from (a) and definition of $L_i$).
\item[(d).] By construction, the set of elements fixed at the end of each iteration partition $E \setminus N_i$, i.e., $E \setminus N_{i} = L_1 \dot{\cup} \hdots \dot{\cup} L_i$.
\item[(e).] We claim that the set of minimum elements $M_i$ at the end of any iteration $i$ always contains the left-over minimum elements from the previous iteration, i.e., $M_{i-1} \setminus L_{i-1} \subseteq M_i$. This is clear if $L_{i-1} = M_{i-1}$, so consider the case when $L_{i-1} \subset M_{i-1}$. At the beginning of the inner loop of iteration $i$, $M = \arg \min_{e \in N_{i-1}} h_e^\prime(x^{(i-1)}(e)) = M_{i-1} \setminus L_{i-1}$. Subsequently, in the inner loop the set of minimum elements can only increase and thus, $M_i \supseteq M_{i-1} \setminus L_{i-1}$.
\item[(f).] We next show that $\epsilon^{(i-1)} < \epsilon^{(i)}$ for $i \geq 2$. Consider an arbitrary iteration $i \geq 2$. If $L_{i-1} = M_{i-1}$, then $\epsilon^{(i)} = \min_{e \in N_{i-1} = N_{i-2}\setminus L_{i-1}} h_e^\prime(x^{(i)}(e)) {\geq}^{(*)} \min_{e \in N_{i-2} \setminus M_{i-1}} h_e^\prime(x^{(i-1)}(e)) > \min_{e \in M_{i-1}} h_e^\prime(x^{(i)}(e)) = \epsilon^{(i-1)}$ where (*) follows from (b). Otherwise, we have $L_{i-1} \subset M_{i-1}$. This implies that $M_{i-1} \setminus L_{i-1}$ is not tight on $x^{(i-1)}$, and it is in fact equal to $\arg \min_{e \in N_{i-1}} h_e^{\prime} (x^{i-1}(e))$ ($=M$ at the beginning of the inner loop in iteration $i$). As this set is not tight, the gradient value can be strictly increased and therefore $\epsilon^{(i)} > \epsilon^{(i-1)}$.
\end{enumerate}
We next split the proof into two cases (A) and (B), depending on how $x^{(0)}$ is initialized.\\

\noindent
\textbf{(A) Proof for $x^{(0)} = 0$}:
\begin{enumerate}
\item[(g).] We claim that $x^{(i)}(e) = x^{(i-1)}(e)$ for all $e \in N_i \setminus M_i$. This follows since the gradient values, $h_e^{\prime}(x^{i}(e))$, for the edges $e\in N_i \setminus M_i$ are greater than $\epsilon^{(i)}$ and thus remain unchanged from $x^{(i-1)}$ (due to claim (a)).
\item[(h).] We claim that $x^{(i)}(e) = 0$ for $e \in N_{i} \setminus M_i$. Using (e) we get $N_i \setminus M_i = N_{i-1} \setminus M_i \subseteq N_{i-1} \setminus (M_{i-1} \setminus L_i)$. Since $L_i \cap N_{i} = \emptyset$, we get $N_i \setminus M_i \subseteq N_{i-1} \setminus M_{i-1}$.  Now (g) implies that $x^{(i)}(e) = x^{(0)}(e)$ for all $e \in N_i \setminus M_i$.
\item[(i).] We next prove that $1\leq j \leq i$, $x^{(i)}(L_1 \cup \hdots \cup L_j) = f(L_1 \cup \hdots \cup L_j)$. First, for $i=1$ note that $T(x^{(1)})$ can be partitioned into $\{L_1, T(x^{(1)}) \setminus M_1\}$. Since $T(x^{(1)}) \setminus M_1 \subseteq N_1 \setminus M_1$, we get $x^{(1)}(T(x^{(1)}) \setminus M_1) =0$ using (g). Thus, by Lemma \ref{nulledges}, we get that $x^{(1)}$ is tight on $L_1$ as well. 

Next, consider any iteration $i >1$. If $j < i$, then $x^{(j)}(L_1 \cup \hdots \cup L_j) = f(L_1 \cup \hdots \cup L_j)$ by induction. Since $x^{(i)} \geq x^{(j)}$, $x^{(i)}$ must also be tight on $L_1 \cup \hdots \cup L_j$. Note that $T(x^{(i)})$ can be partitioned into $\{\big(T(x^{(i)}) \cap (E \setminus N_{i-1}\big), \big(T(x^{(i)}) \cap (N_{i-1} \setminus M_i)\big), \big(T(x^{(i)}) \cap M_i\big)\} = \{\big(L_1 \cup \hdots \cup L_{i-1}\big), \big(T(x^{(i)}) \cap (N_i \setminus M_i)\big), L_i\}$ using (d). Note that $x^{(i)}$ is zero-value on $N_i \setminus M_i$, from (h). By Lemma (\ref{nulledges}) we get that $x^{(i)}$ is also tight on $\big(L_1 \cup \hdots \cup L_i\big)$. 
\end{enumerate}

Claim (b) implies the termination of the algorithm when for some $t$, $N_{t} = \emptyset$. From claim (d), we have obtained a partition of $E$ into disjoint sets $\{L_1, L_2, \hdots, L_t\}$. From claims (a) and (d), we get $x^{(t)}(e) = \epsilon^{(i)}$ for $e \in L_i$. Claim (f) implies that the partition in the theorem $\{F_1, \hdots, F_k\}$ is identical to the partition obtained via the algorithm $\{L_1, \hdots, L_t\}$ (hence $t = k$). Claim (i) implies that $x^{*} = x^{(t)}$ lies in the face $H_{opt}$ as defined in Theorem \ref{main-optimality}.\\

\noindent
\textbf{(B) Proof for $x^{(0)} \in P(f)$ such that $\nabla h(x^{(0)}) =  c\chi(E)$ for some $c \in \mathbb{R}$}: 

\begin{enumerate}

\item[(g$^\prime$).] We claim that $M_{i} = N_{i-1}$. For iteration $i=1$, $h_e^{\prime}(x^{(1)}(e)) = \epsilon^{(1)}$ for all $e \in E$, since $h_e^{\prime}(x^{(0)}(e)) = c$ for all the edges. Thus, indeed, $M_1 = E = N_0$. For iteration $i>1$, we have $N_{i-1} = N_{i-2} \setminus L_{i-1} = M_{i-1} \setminus L_{i-1}$ (by induction). This implies that $h_e^\prime(x^{(i-1)}(e)) = \epsilon^{(i-1)}$ for all the edges in $N_{i-1}$. Thus, in iteration $i$, all the edges $e$ must again have the same gradient value, $h_e^\prime(x^{(i)}(e))$, due to invariant (a).

\item[(h$^{\prime}$).] We claim that $1\leq j \leq i$, $x^{(i)}(L_1 \cup \hdots \cup L_j) = f(L_1 \cup \hdots \cup L_j)$. First, for iteration $i=1$, since $M_1 = E$ by (g$^\prime$), we have $T(x^{(1)}) = L_1$. So, $x^{(1)}$ is tight on $L_1$.

Next, consider any iteration $i >1$. If $j < i$, then $x^{(j)}(L_1 \cup \hdots \cup L_j) = f(L_1 \cup \hdots \cup L_j)$ by induction. Since $x^{(i)} \geq x^{(j)}$, $x^{(i)}$ must also be tight on $L_1 \cup \hdots \cup L_j$. Note that $T(x^{(i)})$ can be partitioned into $\{\big(T(x^{(i)}) \cap (E \setminus N_{i-1}\big), \big(T(x^{(i)}) \cap (N_{i-1} \setminus M_i)\big), \big(T(x^{(i)}) \cap M_i\big)\} = \{\big(L_1 \cup \hdots \cup L_{i-1}\big), \emptyset, L_i\}$ using (d), the claim follows. 

\end{enumerate}

Claim (b) implies the termination of the algorithm when for some $t$, $N_{t} = \emptyset$. From claim (d), we have obtained a partition of $E$ into disjoint sets $\{L_1, L_2, \hdots, L_t\}$. From claims (a) and (d), we get $x^{(t)}(e) = \epsilon^{(i)}$ for $e \in L_i$. Claim (f) implies that the partition in the theorem $\{F_1, \hdots, F_k\}$ is identical to the partition obtained via the algorithm $\{L_1, \hdots, L_t\}$ (hence $t = k$). Claim (h$^{\prime}$) implies that $x^{*} = x^{(t)}$ lies in the face $H_{opt}$ as defined in Theorem \ref{main-optimality}.
\end{proof}

\subsection{Bregman Projections under the Entropy and Euclidean Mirror Maps}\label{bregman}

We next discuss the application of {\sc Inc-fix} algorithm to two important mirror maps. One feature that is common to both is that the trajectory of $x$ in {\sc Inc-Fix} is piecewise linear, and the main step is find the maximum possible increase of $x$ in a given direction $d$. 

\paragraph{\textbf{Unnormalized entropy mirror map}}
This map is given by $\omega(x) = \sum_{e \in E} x(e) \ln x(e) - \sum_{e \in E} x(e)$ and its divergence is $D_{\omega}(x,y) = \sum_{e \in E} x(e) \ln (x(e)/y(e)) - \sum_{e \in E} x(e) + \sum_{e\in E} y(e)$. Note that $\nabla D_{\omega} (x,y) = \ln (\frac{x}{y})$ for a given $y>0$. Finally, $\nabla D_{\omega}^{-1}(x) = y e^{-x}$. In order to minimize the divergence $D_{\omega}(x,y)$ with respect to a given point $y \in \mathbb{R}_+^E$ using the {\sc Inc-Fix} algorithm, we need to find the maximum possible increase in the gradient $\nabla D_{\omega}$ while remaining in the submodular polytope $P(f)$. In each iteration, this amounts to computing:
\begin{align*}
\epsilon_1 &= \max \{ \delta : (\nabla D_{\omega})^{-1}(\nabla D_{\omega} (x) + \delta \chi(M)) \in P(f)\} = \max \{ \delta : y e^{(\nabla D_{\omega} (x) + \delta \chi(M))} \in P(f)\}\\
& = \max \{ \delta : x + z \in P(f), z(e) = (e^{\delta}-1)x(e) \textrm{ for } e\in M,  z(e) = 0 \textrm{ for } e \notin M\}, 
\end{align*}
for some $M \subseteq E$. For the entropy mirror map, the point to be projected, $y$, must be positive in each coordinate (otherwise the divergence is undefined). We initialize $x^{(0)} = cy \in P(f)$ for a small enough constant $c >0$. Apart from satisfying the conditions of the {\sc Inc-Fix} algorithm, this ensures that there is a well defined direction for increase in each iteration. 

\paragraph{\textbf{Euclidean mirror map}}
The Euclidean mirror map is given by $\omega(x) = \frac{1}{2} ||x||^2$ and its divergence is $D_{\omega}(x,y) = \frac{1}{2} ||x-y||^2$. Here, for a given $y \in \mathbb{R}$, $\nabla D_{\omega}(x,y) = \nabla \omega(x) - \nabla \omega (y) = x - y$. This implies that for any iteration, we need to compute 
\begin{align*}
\epsilon_1 &= \max \{ \delta: (\nabla D_{\omega})^{-1}(\nabla D_{\omega} (x) + \delta \chi(M)) \in P(f)\} = \max \{\delta: (\nabla D_{\omega} (x) + \delta \chi(M)) + y \in P(f)\}\\
& = \max \{\delta: x + \delta \chi(M) \in P(f)\},
\end{align*} 
for some $M \subseteq E$. Notice that in each iteration of both the algorithms for the entropy and the Euclidean mirror maps, we need to find the maximum possible increase to the value on non-tight elements in a fixed given direction $d \geq 0$. For the case of the entropy mirror map, $d_e = x_e$ for $e\in M$ and $d_e=0$ otherwise, while, for the Euclidean mirror map, $d = \chi(M)$ (for current iterate $x \in P(f)$, and corresponding minimum weighted set of edges $M \subseteq E$). 

Let {\sc Line} be the problem of finding the maximum $\delta$ s.t. $x + \delta d \in P(f)$. $P_1$ is equivalent to finding maximum $\delta$ such that $\min_{S \subseteq E} f(S) - (x+\delta d)(S)=0$; this is a parametric submodular minimization problem. For the graphic matroid, feasibility in the forest polytope can be solved by $O(|V|)$ maximum flow problems (\cite{Cunningham1985}, Corollary 51.3a in \cite{Schrijver}), and as result, {\sc Line} can be solved as $O(|V|^2)$ maximum flow problems (by Dinkelbach's discrete Newton method) or  $O(|V|)$ parametric maximum flow problems. 
For general polymatroids over the ground set $E$, the problem {\sc Line} can be solved using Nagano's parametric submodular function minimization (\cite{Nagano2007f}) that requires $O(|E|^6 + \gamma |E|^5)$ running time, where $\gamma$ is the time required by the value oracle of the submodular function. Each of the entropy and the Euclidean mirror maps requires $O(|E|)$ ($O(|V|)$ for the graphic matroid) such computations to compute a projection, since each iteration of the {\sc Inc-Fix} algorithm at least one non-tight edge becomes tight. Thus, for the graphic matroid we can compute Bregman projections in $O(|V|^4|E|)$ time (using Orlin's $O(|V||E|)$ algorithm (\cite{Orlin2013}) for computing the maximum flow) and for general polymatroids the running time is $O(|E|^7 + \gamma |E|^6)$ where $|E|$ is the size of the ground set. For cardinality-based submodular functions\footnote{A submodular function is cardinality-based if $f(S) = g(|S|)$ for all $S \subseteq E$ and some concave $g: \mathbb{N} \rightarrow \mathbb{R}$.}, we can prove a better bound on the running time. 
 
\begin{lemma}\label{cardinality} The {\sc Inc-Fix} algorithm takes $O(|E|^2)$ time to compute projections over base polytopes of polymatroids when the corresponding submodular function is cardinality-based. 
\end{lemma} 
\begin{proof} Let $f$ be a cardinality-based submodular function such that $f(S) = g(|S|)$ for all $S \subseteq E$ and some concave $g: \mathbb{N} \rightarrow \mathbb{R}$. Let $T(x)$ denote the maximal tight set of $x \in P(f)$. Let {\sc Line} be the problem to find $\lambda^* = \max \{\lambda: x + \lambda z \in P(f)\}$ where $x \in P(f), z \in \mathbb{R}_+^{E}$. Let $S(x)$ be a sorted sequence of components of $x$ in decreasing order. We will show that for both the mirror maps, we can solve {\sc Line} in $O(|E|)$ time.

\begin{itemize}
\item[(i)] We first claim that $(x + \lambda^*y)(e) \leq \min_{e^\prime \in T(x)} x(e^\prime)$ for all $e \in E \setminus T(x)$. Let $e^* \in \arg \min_{e^\prime \in T(x)} x(e^\prime)$. The claim holds since otherwise the submodular constraint on the set $T(x) \setminus \{e^*\} \cup \{e^\prime\}$ (which is of the same cardinality as $T(x)$) will be violated. 

\item[(ii)] Let $S(x) = \{x(e_1), x(e_2), \hdots, x(e_m)\}$ such that $x(e_i) \geq x(e_j)$ for $1\leq i < j \leq m$. Then, to check if a vector $x$ is in $P(f)$, one can simply check if $\sum_{i=1}^k x(e_i) \leq g(k)$ for each $k = 1, \hdots, m$ as the submodular function $f$ is cardinality-based. 

\item[(iii)] We next claim that for each of the Euclidean and entropy mirror maps, the ordering of the elements remains the same in subsequent iterations, i.e., at the end of each iteration $i$, the ordering $S(x^{(i)}) = S(x^{(i-1)})$ up to equal elements. 

\begin{itemize}
\item[(A)] For the Euclidean mirror map, with $x^{(0)} = 0$: We first claim that $S(y) = S(x^{(0)})$ up to equal elements; this holds since $x^{(0)}=0$. Next, for each iteration $i$ we claim that $S(x^{(i-1)}) = S(x^{(i-1)} + \lambda z)$ up to equal elements for $\lambda>0$. For the Euclidean mirror map, $z = \chi(M)$ for $M = \arg \min_{e \in N_{i-1}} (x^{(i-1)}(e) - y(e))$ (or for some intermediate iterate in the inner loop). Let $D = \{e \in N_{i-1}: x^{(i-1)}(e) > 0\}$. Note that $D \subseteq M$  using claim (e) of the proof for Theorem \ref{inc-fix-correctness}. As all elements of $D$ are increased by the same amount $\lambda$, while being bound by the value of minimum element in $T(x)$, their ordering remains the same as $S(x^{(i-1)})$ after the increase. Further the zero-elements in $M$, i.e., $M \setminus D$ have the highest value of $y(e)$ among zero-elements of $x^{(i-1)}$. Therefore, increasing these uniformly respects the ordering with respect to $S(y)$, while being less than the value of elements in $D$.


\item[(B)] For the entropy mirror map, with $x^{(0)} = \epsilon y \in P(f)$:

For each iteration $i$, we claim that $S(x^{(i-1)}) = S(x^{(i-1)} + \lambda z)$ up to equal elements, for some $\lambda>0$. For the entropy mirror map, the direction $z$ is given by $z(e) = x^{(i)}(e)$ for $e \in N_{i-1} = E \setminus T(x)$ (see claim $(g^\prime)$ in the proof for Theorem \ref{inc-fix-correctness}) and $z(e) = 0$ for $e \in T(x)$. Since we increase $x^{(i-1)}$ proportionally to $x^{(i-1)}(e)$ on all elements in $N_{i-1}$, while being bound by the value of the minimum element in $T(x)$, the ordering of the elements $S(x^{(i-1)} + \lambda z)$ is the same as $S(x^{(i-1)})$ (up to equal elements). 
\end{itemize}

\item[(iv)] Since the ordering of the elements remains the same after an increase, we get an easy way to solve the problem {\sc LINE}. For each $k = |T(x)|+1, \hdots, |E|$, we can compute the maximum possible increase possible without violating a set of cardinality $k$ which is given by $t_k = \frac{g(k) - \sum_{i=1}^{k} x_i}{\sum_{i=|T(x)|}^k z(i)}$. Then $\lambda^* = \min_{k} t_k$ and this can be checked in $O(|E|)$ time. 
\end{itemize}

Hence, we require a single sort at the beginning of the {\sc Inc-Fix} algorithm ($O(|E| \ln |E|)$ time), and using this ordering (that does not change in subsequent iterations) we can perform {\sc Line} in $O(|E|)$ time. Therefore, the running time of {\sc Inc-Fix} for cardinality-based submodular functions is $O(|E|^2)$. 
\end{proof}

\paragraph{\textbf{Computing Nash-equilibria}} Let us now consider the MSP game over strategy polytopes $P$ and $Q$ under a bilinear loss function, such that $P$ is the base polytope of a matroid (and $Q$ is any polytope that one can optimize linear functions over). The online mirror descent algorithm starts with $x^{(0)}$ being the $\omega-$center of the base polytope, that is simply obtained by projecting a vector of ones on the base polytope. Each subsequent iteration $x^{(t)}$ is obtained by projecting an appropriate point $y^{(t)}$ under the Bregman projection of the mirror map. For the entropy map, $y^{(t)} = [x^{(t-1)}_1e^{- \eta \nabla l^{(t-1)}_1}; \hdots; x^{(t-1)}_me^{- \eta \nabla l^{(t-1)}_m}]$, and for the Euclidean map $y^{(t)}$ is simply given by $y^{(t)} = [x^{(t-1)}_1 - \eta \nabla l^{(t-1)}_1; \hdots; x^{(t-1)}_m - \eta \nabla l^{(t-1)}_m]$. Here, $\nabla l^{(t)} = Lv^{(t)}$ where $v^{(t)} = \arg\max_{z \in Q} x^{(t)T}L z$. Note that for the entropy map, each $y^{(t)}$ (that we project) is guaranteed to be strictly greater than zero since $x^{(0)}>0$. Assuming the loss functions are $G$-Lipschitz under the appropriate norm (i.e., $L_1$-norm for the entropy mirror map, and $L_2$-norm for the Euclidean mirror map), after $T = O(G^2R^2/\epsilon^2)$ iterations of the mirror descent algorithm, we obtain a $\epsilon-$approximate Nash-equilibrium of the MSP game given by $(\sum_{i=1}^T x^{(i)}/T, \sum_{i=1}^T v^{(i)}/T)$. For the entropy map, $R^2 \leq r(E) \ln (m)$ and for the Euclidean mirror map, $R^2 \leq r(E)$. Using the Euclidean mirror map (as opposed to the entropy map) even though we reduce the $R^2$ term in the number of iterations, the Lipschitz constant might be greater with respect to the $L_2$-norm (as opposed to the $L_1$-norm). For example, suppose in a spanning tree game the loss matrix $L$ is scaled such that $||L||_{\infty} \leq 1$. Then, the loss functions are such that $G_{\mathrm{entropy}} = ||\nabla l^{(i)}||_\infty = ||L v^{(i)}||_{\infty} \leq n$ and $G_{\mathrm{Euc}} = ||\nabla l_i||_2 = ||L v^{(i)}||_{2} \leq n\sqrt{m}$ and so, the online mirror descent algorithm converges to an $\epsilon$-approximate strategy in $O(R^2G^2/\epsilon^2) = O(n^2 \; r(E) \ln m/\epsilon^2) = O(n^3 \ln m/\epsilon^2)$ rounds (of learning) under the entropy mirror map\footnote{Even though this case is identical to the multiplicative weights update algorithm, the general analysis for the mirror descent algorithms gives a better convergence rate with respect to the size of the graph $n, m$ (but the same dependence on $\epsilon$).}, whereas it takes $O(n^3 m /\epsilon^2)$ rounds under the Euclidean map.

\section{The Multiplicative Weights Update Algorithm}\label{mwu}
We now restrict our attention to MSP games over 0/1 strategy polytopes $P$ and $Q$ such that $\mathcal{U} =$ vert$(P) \subseteq \{0,1\}^m$ and $\mathcal{V} =$ vert$(Q) \subseteq \{0,1\}^n$. The vertices of these polytopes constitute the pure strategies of these games (i.e., combinatorial concepts like spanning trees, matchings, k-sets). We review the Multiplicative Weights Update (MWU) algorithm for MSP games over strategy polytopes P and Q. The {\sc MWU} algorithm starts with the uniform distribution over all the vertices, and simulates an iterative procedure where the learner (say player 1) plays a mixed strategy $x^{(t)}$ in each round $t$.  In response the {\sc Oracle} selects the most adversarial loss vector for the learner, i.e., $l^{(t)} = Lv^{(t)}$ where $v^{(t)} = \arg \max_{y \in Q} x^{(t)}Ly$. The learner observes losses for all the pure strategies and incurs loss equal to the expected loss of their mixed strategy. Finally the learner updates their mixed strategy by lowering the weight of each pure strategy $u$ by a factor of $\beta^{u^Tl^{(t)}/F}$ for a fixed constant $0< \beta < 1$ and a factor $F$ that accounts for the magnitude of the losses in each round. That is, for each round $t \geq 1$, the updates in the MWU algorithm are as follows, starting with $w^{(1)}(u)  = 1$ for all $u \in \mathcal{U}$: 
\begin{align*}
x^{(t)} &= \frac{\sum_{u \in \mathcal{U}} w^{(t)}(u) u}{\sum_{u \in \mathcal{U}} w^{(t)}(u)}, v^{(t)} = \arg\max_{y \in Q} x^{(t)T}Ly, 
w^{(t+1)}(u) &= w^{(t)}(u) \beta^{u^{T}Lv^{(t)}/F} \;\; \forall u \in \mathcal{U}.
\end{align*}
Standard analysis of the MWU algorithm shows that an $\epsilon-$approximate Nash-equilibrium can be obtained in $O((\frac{\ln |\mathcal{U}|}{(\epsilon/F)^2})$ rounds in the case of MSP games (see for e.g. \cite{Arora2012}). 
 
Note that in many interesting cases of MSP games, the input of the game is $O(\ln |\mathcal{U}|)$. Even though the MWU algorithm converges in $O(\ln |\mathcal{U}|)$ rounds it requires $O(|\mathcal{U}|)$ updates per round. We will show in the following sections how this algorithm can be simulated in polynomial time (i.e., polynomial in the input of the game). 

\subsection{MWU in Polynomial Time}\label{polynomial-mwu}
We show how to simulate the MWU algorithm in time polynomial in $\ln |\mathcal{U}|$ where $\mathcal{U}$ is the vertex set of the 0/1 polytope $P \subset \mathbb{R}^m$, by the use of \emph{product distributions}. A product distribution $p$ over the set $\mathcal{U}$ is such that $p(u) \propto \prod_{e \in u} \lambda_{e}$ for some vector $\lambda \in \mathbb{R}^m$. We refer to the $\lambda$ vector as the \emph{multiplier vector} of the product distribution. The two key observations here are that \emph{product distributions can be updated efficiently by updating only the multipliers (for bilinear losses)}, and \emph{multiplicative updates on a product distribution results in a product distribution again}. 

To argue that the MWU can work by updating only product distributions, suppose first that in some iteration $t$ of the MWU algorithm, we are given a product distribution $p^{(t)}$ over the vertex set $\mathcal{U}$ implicitly by its multiplier vector $\lambda^{(t)}$, and a loss vector $l^{(t)} \in \mathbb{R}^m$ such that the loss of each vertex $u$ is $u^Tl^{(t)}$. In order to multiplicatively update the probability of each vertex $u$ as $p^{(t+1)}(u) \propto p^{(t)}(u) \beta^{u^Tl^{(t)}}$, note that we can simply update the multipliers with the loss of each component. 
\begin{align}
p^{(t+1)} (u) &\propto p^{(t)}(u) \beta^{u^Tl^{(t)}}\nonumber \propto \left(\prod_{e \in u} \lambda^{(t)} (e)\right) \beta^{u^Tl^{(t)}} \propto \prod_{e \in u} \left(\lambda^{(t)} (e) \beta^{l^{(t)}(e)}\right) && \textrm{ as } u \in \{0,1\}^m.
\end{align}

Hence, the resulting probability distribution $p^{(t+1)}$ is also a product distribution, and we can implicitly represent it in the form of the multipliers $\lambda^{(t+1)}(e) = \lambda^{(t)}(e)\beta^{l^{(t)}(e)} (e \in [m])$ in the next round of the MWU algorithm. 

Suppose we have a {\it generalized counting oracle} \textbf{M} which, given $\lambda\in \mathbb{R}^m_+$,  computes $\sum_{u\in {\mathcal U}}\prod_{e: u_e=1} \lambda(e)$ and also, for any element $f$, computes $\sum_{u\in {\mathcal U}: u_f=1}\prod_{e: u_e=1} \lambda(e)$. Such an oracle can be used to compute  the marginals $x \in P$ corresponding to the product distribution associated with $\lambda$.
Suppose we have also an \emph{adversary oracle} \textbf{R} that computes the worst-case response of the adversary given a marginal point in the learner's strategy polytope, i.e., \textbf{R}$(x) = \arg\max_{v \in \mathcal{V}} x^TLv$ (the losses depend only on marginals of the learner's strategy and not the exact probability distribution). Then, we can exactly simulate the MWU algorithm. We can initialize the multipliers to be $\lambda^{(1)}(e) = 1$ for all $e \in [m]$, thus effectively starting with uniform weights $w^{(1)}$ across all the vertices of the polytope. Given the multipliers $\lambda^{(t)}$ for each round, we can compute the corresponding marginal point $x^{(t)}=\mathbf{M}(\lambda^{(t)})$ and the corresponding loss vector $Lv^{(t)}$ where $v^{(t)} = \mathbf{R}(x^{(t)})$. Finally, we can update the multipliers with the loss in each component, as discussed above. It is easy to see that the standard proofs of convergence of the MWU go through, as we only change the way of updating probability distributions, and we obtain the statement of Theorem \ref{MWU-main-exact}. We assume here that the loss matrix $L \geq 0\footnote{This is however an artificial condition, and the regret bounds change accordingly for general losses.\label{footnote1}}$. The proof is included in Appendix \ref{app:mwu}.

\begin{theorem} \label{MWU-main-exact} Consider an MSP game with strategy polytopes $P$, $Q$ and $loss(x,y) = x^TLy$ for $x \in P, y \in Q$, as defined above. Let $F = \max_{x \in P, y \in Q} x^TLy$ and $\mathcal{U} = \textrm{vert}(P)$. Given two polynomial oracles \textbf{M} and \textbf{R} where \textbf{M}$(\lambda) = x$ is the marginal point corresponding to multipliers $\lambda$, and \textbf{R}$(x) = \arg\max_{y \in Q} x^TLy$, the algorithm {\sc MWU} with product updates gives an $O(\epsilon)-$approximate Nash equilibrium ($\bar{x}, \bar{y}$) = $(\frac{1}{t}\sum_{i=1}^t x^{(i)}, \frac{1}{t}\sum_{i=1}^t v^{(i)})$ in $O(\frac{\ln(|\mathcal{U}|)}{(\epsilon/F)^2})$ rounds and time polynomial in $(n,m)$.
\end{theorem}

Having shown that the updates to the weights of each pure strategy can be done efficiently, the regret bound follows from known proofs for convergence of the multiplicative weights update algorithm. We would like to draw attention to the fact that, by the use of product distributions, we are not restricting the search of approximate equilibria. This follows from the analysis, and also from the fact that any point in the relative interior of a 0/1 polytope can be viewed as a (max-entropy) product distribution over the vertices of the polytope (\cite{Asadpour}, \cite{Singh2014}). 

\paragraph{\textbf{Approximate Computation}.} If we have an {\it approximate} generalized counting oracle, we would have an approximate marginal oracle \textbf{M}$_{\epsilon_1}$ that computes an estimate of the marginals, i.e. \textbf{M}$_{\epsilon_1}(\lambda) = \tilde{x}$ such that $||M(\lambda) - \tilde{x}||_{\infty} \leq \epsilon_1$ where $M(\lambda)$ is the true marginal point corresponding to $\lambda$, and an approximate adversary oracle \textbf{R}$_{\epsilon_2}$ that computes an estimate of the worst-case response of the adversary given a marginal point in the learner's strategy polytope, i.e. \textbf{R}$_{\epsilon_2}(x) = \tilde{v}$ such that $x^TL \tilde{v} \geq \max_{v \in \mathcal{V}} x^TL v - \epsilon_2$ (for example, in the case when the strategy polytope is not in $P$ and only an FPTAS is available for optimizing linear functions). 
\begin{algorithm}
\caption{The {\sc MWU} algorithm with approximate oracles}
\label{MWU-algorithm-approx}
\Input{\textbf{M}$_{\epsilon_1}$: $\mathbb{R}^m \rightarrow \mathbb{R}^m$, \textbf{R}$_{\epsilon_2}$: $\mathbb{R}^m \rightarrow \mathbb{R}^n$, $\epsilon>0.$}
\Output{$O(\epsilon+F\epsilon_1+\epsilon_2)$-approximate Nash equilibrium $(\bar{x},\bar{y})$}
{$\lambda^{(1)} = \mathbf{1}, t=1, F = \max_{x\in P, y \in Q} x^TLy, \epsilon^\prime=\epsilon/F, \beta = \frac{1}{1+\sqrt{2}\epsilon^\prime}$\;}
 \Repeat{$t < \frac{F^2\ln |U|}{\epsilon^2}$}{
 	$\tilde{x}^{(t)} =$ \textbf{M}$_{\epsilon_1}(\lambda^{(t)})$\;
	$\tilde{v}^{(t)} =$ \textbf{R}$_{\epsilon_2}(\tilde{x}^{(t)})$\;
	$\lambda^{(t+1)}(e) = \lambda^{(t)}(e) * \beta^{L\tilde{v}^{(t)}(e)/F} \;\; \forall e \in E$\;
	$t \leftarrow t+1$\;
}
$(\bar{x}, \bar{y}) = (\frac{1}{t-1}\sum_{i=1}^{t-1} \tilde{x}^{(i)}, \frac{1}{t-1}\sum_{i=1}^{t-1} \tilde{v}^{(i)})$\;
\vspace{0.1cm}
\end{algorithm}
We give a complete description of the algorithm in Algorithm \ref{MWU-algorithm-approx} and the formal statement of the regret bound in the following lemma (for loss matrices $L \geq 0^{\ref{footnote1}}$) (proved in Appendix \ref{app:mwu}). The tricky part in the proof is that since the loss vectors are approximately computed from approximate marginal points, there is a possibility of not converging at all. However, we show that this is not the case since we maintain the true multipliers $\lambda^{(t)}$ in each round. It is not clear if there would be convergence, for example, had we gone back and forth between the marginal point and the product distribution. 

\begin{lemma} \label{MWU-main} Given two polynomial approximate oracles \textbf{M}$_{\epsilon_1}$ and \textbf{R}$_{\epsilon_2}$ where \textbf{M}$_{\epsilon_1}(\lambda) = \tilde{x}$ s.t. $||M(\lambda) - \tilde{x}||_{\infty} \leq \epsilon_1$, and \textbf{R}$_{\epsilon_2}(x) = \tilde{v}$ s.t. $x^TL\tilde{v} \geq \max_{y \in Q} x^TLy - \epsilon_2$, the algorithm {\sc MWU} with product updates gives an $O(\epsilon + F\epsilon_1 + \epsilon_2)-$approximate Nash equilibrium ($\bar{x}, \bar{y}$) = $(\frac{1}{t}\sum_{i=1}^t\tilde{x}^{(i)}, \frac{1}{t}\sum_{i=1}^t\tilde{v}^{(i)})$ in $O(\frac{F^2 \ln(|\mathcal{U}|)}{\epsilon^2})$ rounds and time polynomial in $(n,m)$.
\end{lemma}

\paragraph{\bf Applications:} 
For learning over the spanning tree polytope, an exact generalized counting algorithm follows from Kirchoff's  matrix theorem (\cite{lyons2005probability}) that states that the value of $\sum_{T} \prod_{e \in T} \lambda_e$ is equal to the value of the determinant of any cofactor of the weighted Laplacian of the graph. One can use fast Laplacian solvers (see for e.g., \cite{Koutis2010}) for obtaining a fast approximate marginal oracle. Kirchhoff's determinantal formula also extends to (exact) counting of bases of regular matroids. For learning over the bipartite matching polytope (i.e., rankings), one can use the randomized generalized approximate counting oracle from (\cite{Jerrum2004}) for computing permanents to obtain a feasible marginal oracle. Note that the problem of counting the number of perfect matchings in a bipartite graph is \#P-complete as it is equivalent to computing the permanent of a 0/1 matrix (\cite{Valiant1979}). The problem of approximately counting the number of perfect matchings in a general graph is however a long standing open problem, if solved, it would result in another way of solving MSP games on the matching polytope. Another example of a polytope that admits a polynomial approximate counting oracle is the cycle cover polytope (or $0-1$ circulations) over directed graphs (\cite{Singh2014}). Also, we would like to note that to compute Nash-equilibria for MSP games that admit marginal oracles for both the polytopes, the optimistic mirror descent algorithm is simply the exponential weights with a modified loss vector (\cite{Rakhlin2013}), and hence the same framework applies.

\paragraph{\bf Sampling pure strategies:} In online learning scenarios that require the learner to play a combinatorial concept (i.e., a pure strategy in each round), we note first that given any mixed strategy (that lies in a strategy polytope $\in \mathbb{R}^n$), the learner can obtain a convex decomposition of the mixed strategy into at most $n+1$ vertices by using the well-known Caratheodory's Theorem. The learner can then play a pure strategy sampled proportional to the convex coefficients in the decomposition. In the case of learning over the spanning tree and bipartite perfect matching polytopes using product distributions however, there exists a more efficient way of sampling due to the \emph{self-reducibility}\footnote{Intuitively, self-reducibility means that there exists an inductive construction of the combinatorial object from a smaller instance of the same problem. For example, conditioned on whether an edge is taken or not, the problem of finding a spanning tree (or a matching) on a given graph reduces to the problem of finding a spanning tree (or a matching) in a modified graph.} of the these polytopes (\cite{Kulkarni1990}, \cite{Asadpour}): Order the edges of the graph randomly and decide for each probabilistically whether to use it in the final object or not. The probabilities of each edge are updated after every iteration conditioned on the decisions (i.e., to include or not) made on the earlier edges. This sampling procedure works as long as there exists a polynomial time marginal oracle (i.e., a generalized counting oracle) to update the probabilities of the elements of the ground set after each iteration \emph{and} if the polytope is \emph{self-reducible} (\cite{Sinclair1989}). Intuitively, self-reducibility means that there exists an inductive construction of the combinatorial object from a smaller instance of the same problem. For example, conditioned on whether an edge is taken or not, the problem of finding a spanning tree (or a matching) on a given graph reduces to the problem of finding a spanning tree (or a matching) in a modified graph. 

\section{Symmetric Nash-equilibria}\label{sne}
In this section we explore purely combinatorial algorithms to find Nash-equilibria, without using learning algorithms. Symmetric Nash-equilibria are a set of optimal strategies such that both players play the exact same mixed strategy at equilibrium. We assume here that the strategy polytopes of the two players are the same. In this section, we give necessary and sufficient conditions for a symmetric Nash-equilibrium to exist in case of matroid MSP games. More precisely, our main result is the following:

\begin{theorem} \label{symmetric-general} Consider an MSP game with respect to a matroid $M = (E, \mathcal{I})$ with an associated rank function $r: E \rightarrow \mathbb{R}_+$. Let $L$ be the loss matrix for the row player such that it is symmetric, i.e. $L^T = L$. Let $x \in B(M) = \{x \in \mathbb{R}^\mathrm{E}: x(S) \leq r(S) \;\forall\; S \subset E, x(E) = r(E), x \geq 0\}$. Suppose $x$ partitions the elements of the ground set into $\{P_1, P_2, \hdots P_k\}$ such that $(Lx)(e) = c_i \;\forall e \in P_i$ and $c_1 < c_2 \hdots < c_k$. Then, the following are equivalent. 
\begin{enumerate}
\item [(i).] $(x,x)$ is a symmetric Nash-equilibrium, 
\item [(ii).] All bases of matroid $M$ have the same cost with respect to weights $Lx$,
\item [(iii).] For all bases $B$ of $M$, $|B \cap P_i| = r(P_i)$,
\item [(iv).] $x(P_i) = r(P_i)$ for all $i \in \{1, \hdots, k\}$,
\item [(v).] For all circuits $C$ of $M$, $\exists i : C \subseteq P_i$. 
\end{enumerate}
\end{theorem}
\begin{proof}
\noindent
Case (i) $\Leftrightarrow$ (ii). Assume first that $(x,x)$ is a symmetric Nash-equilibrium. Then, the value of the game is $\max_{z \in B(M)} x^TLz = \min_{z \in B(M)} z^TLx = \min_{z \in B(M)} x^TL^Tz \overset{(1)}{=} \min_{z \in B(M)} x^TLz$, where (1) follows from $L^T=L$. This implies that every base of the matroid has the same cost under the weights $Lx$.\\
Coversely, if every base has the same cost with respect to weights $Lx$, then $x \in \arg\max_{y \in B(M)} x^TLy$ and $x \in \arg\min_{y \in B(M)} x^TLy$. Since no player has an incentive to deviate, this implies that $(x,x)$ is a Nash-equilibrium. 
\newline\newline
Case (ii) $\Leftrightarrow$ (iii). Assume (ii) holds. Suppose there exists a base $B$ such that $|B \cap P_i| < r(P_i)$ for some $i$. We know that there exists a base $B^{T}$ such that $|B^{T} \cap P_i |= r(P_i)$. Since $B \cap P_i, B^{T} \cap P_i \in \mathcal{I}$ and $|B^{T} \cap P_i| > |B \cap P_i|$, $\exists e \in (B^{T} \setminus B) \cap P_i$ such that $(B \cap P_i) + e \in \mathcal{I}$. Since $(B \cap P_i) + e \in \mathcal{I}$ and $B$ is a base, $\exists f \in B \setminus P_i$ such that $B + e - f \in \mathcal{I}$. This gives a base of a different cost as $e$ and $f$ are in different members of the partition. Hence, we reach a contradiction. Thus, (ii) implies (iii).\\
Conversely, assume (iii) holds. Note that the cost of a base B is $c(B) = \sum_{i=1}^k c_i|P_i \cap B|$. Thus, (iii) implies that every base has the same cost $\sum_{i=1}^k c_i r(P_i)$.
\newline\newline
Case (iii) $\Leftrightarrow$ (iv). Assume (iii) holds. Since $x \in B(M)$, $x$ is a convex combination of the bases of the matroid, i.e. $x = \sum_{B} \l_B \chi(B)$ where $\chi(B)$ denotes the characteristic vector for the base $B$. Thus, (iii) implies that $x(P_i) = \sum_B \lambda_B |B \cap P_i| = \sum_{B} \lambda_B r(P_i) = r(P_i)$ for all $i \in \{1, \hdots, k\}$.\\
Conversely assume (iv) and consider any base B of the matroid. Then, $r(E) = |B| = \sum_{i=1}^k |B \cap P_i| {\leq}^{(1)} \sum_{i=1}^k r(P_i)$ $\overset{(2)}{=} \sum_{i=1}^k x(P_i) = x(E) = r(E)$, where (1) follows from rank inequality and (2) follows from (iv) for each $P_i$. Thus, equality holds in (1) and we get that for each base $B$, $|B \cap P_i| = r(P_i)$.
\newline\newline
Case (iii) $\Leftrightarrow$ (v). Assume (iii) and let $C$ be a circuit. Let $e \in C$ and $B$ be a base that contains $C - e$. Hence, the unique circuit in $B + e$ is $C$. Thus, for any element $f \in C -e$, $B - e + f \in \mathcal{I}$. Hence, (iii) implies that all the elements of $C - e$ must lie in the same member of the partition as $e$ does. Hence, $\exists i: C \subseteq P_i$.\\
Conversely, assume (v). Consider any two bases $B$ and $B^T$ such that $B \setminus B^T = \{e\}$ and $B^T - B = \{f\}$ for some $e,f \in E$. Let $C$ be the unique circuit in $B^T + e$ and hence $f \in C$. It follows from (v) that $e,f$ are in the same member of the partition, and hence $|B \cap P_i| = |B^T \cap P_i|$ for all $i \in \{1,\hdots, k\}$. Since we know there exists a base $B_i$ such that $|B_i \cap P_i| = r(P_i)$ for each $i$, hence all bases must have the same intersection with each $P_i$ and (iii) follows. 
\end{proof}
\begin{corollary} Consider a game where each player plays a base of the graphic matroid $M(G)$ on a graph $G$, and the loss matrix of the row player is the identity matrix $I \in \mathbb{R}^{\mathrm{E} \times \mathrm{E}}$. Then there exists a symmetric Nash-equilibrium \emph{if and only if} every block of $G$ is uniformly dense. 
\end{corollary}
\begin{proof} Since the loss matrix is the identity matrix, $x(e) = c_i$ for all $e \in P_i$. Theorem \ref{symmetric-general} (v) implies that each $P_i$ is a union of blocks of the graph. Further, as $x(P_i) = r(P_i) = c_i|P_i|$, each $P_i$ (and hence each block contained in $P_i$) is uniformly dense. 
\end{proof}

\begin{corollary} Given any point $x \in \mathbb{R}^\mathbb{E}, x>0$ in the base polytope of a matroid $M = (E, \mathcal{I})$, one can construct a matroid game for which $(x,x)$ is the symmetric Nash equilibrium.
\end{corollary}
\begin{proof} Let the loss matrix $L$ be defined as $L_{e,e} = 1/x_{e}$ for $e \in E$ and 0 otherwise. Then, $Lx(e) = 1$ for all $e \in E$. Thus, all the bases have the same cost under $Lx$. It follows from Theorem \ref{symmetric-general} that $(x,x)$ is a symmetric Nash equilibrium of this game.  
\end{proof}

\noindent \textbf{Uniqueness: } Consider a symmetric loss matrix $L$ such that $L_{e,f} = 1$ for all $e,f\in E$. Note that any feasible point in the base polytope $B(M)$ forms a symmetric Nash equilibrium. Hence, we need a stronger condition for the symmetric Nash equilibria to be unique. We note that for positive and negative-definite loss matrices, symmetric Nash-equilibria are unique, if they exist (proof in the Appendix).  
\begin{theorem} \label{unique} Consider the game with respect to a matroid $M = (E, \mathcal{I})$ with an associated rank function $r: E \rightarrow \mathbb{R}_+$. Let $L$ be the loss matrix for the row player such that it is positive-definite, i.e. $x^TLx > 0$ for all $x \neq 0$. Then, if there exists a symmetric Nash equilibrium of the game, it is unique.
\end{theorem}

\begin{proof} Suppose $(x,x)$ and $(y,y)$ are two symmetric Nash equilibria such that $x \neq y$, then the value of the game is $x^TLx = y^TLy$. Then,  $x^TLz \leq x^TLx \leq z^TLx$ $\forall z \in B(M)$ implying that $x^TLz \leq x^TLx \leq x^TL^Tz$ $\forall z \in B(M)$. Since $L$ is symmetric, we get $x^TLx = x^TLz$ $\forall z \in B(M)$. Similarly, we get $y^TLy = y^TLz$ $\forall z \in B(M)$. Consider $z = \frac{x+y}{2}$. Then, $z^TLz = {\frac{(x+y)}{2}}^TLz = {\frac{x}{2}}^TLz + {\frac{y}{2}}^TLz =  {\frac{x}{2}}^TLx + {\frac{y}{2}}^TLy$. This contradicts the strict convexity of the quadratic form of $x^TLx$. Hence, $x=y$ and there exists a unique symmetric Nash equilibrium.
\end{proof}
\noindent

\noindent \textbf{Lexicographic optimality: } We further note that symmetric Nash-equilibria are closely related to the concept of being lexicographically optimal as studied in \cite{Fujishige1980b}. For a matroid $M = (E, \mathcal{I})$, $x \in B(M)$ is called lexicographically optimal with respect to a positive weight vector $w$ if the $|E|$-tuple of numbers $x(e)/w(e)$ $(e \in E)$ arranged in the order of increasing magnitude is lexicographically maximum among all $|E|$-tuples of numbers $y(e)/w(e)$ $(e \in E)$ arranged in the same manner for all $y \in B(M)$. We evoke the following theorem from \cite{Fujishige1980b}.

\begin{theorem} \label{fujishige-theorem} Let $x\in B(M)$ and $w$ be a positive weight vector. Define $c(e) = x(e)/w(e)$ $(e \in E)$ and let the distinct numbers of $c(e)$ $(e\in E)$ be given by $c_1 < c_2 < \hdots < c_p$. Futhermore, define $S_i \subseteq E \;\; (i = 1,2, \hdots, p)$ by $$S_i = \{ e | e \in E, c(e) \leq c_i\} \;\; (i = 1,2, \hdots, p).$$  
Then the following are equivalent:
\begin{enumerate}
\item[(i)] $x$ is the unique lexicographically optimal point in $B(M)$ with respect to the weight vector $w$;
\item[(ii)] $x(S_i) = r(S_i) \;\; (i = 1, 2, \hdots, p).$
\end{enumerate}
\end{theorem}

The following corollary gives an algorithm for computing symmetric Nash-equilibria for matroid MSP games. 
\begin{corollary} \label{lex-opt-equal}
Consider a matroid game with a diagonal loss matrix $L$ such that $L_{e,e}>0$ for all $e \in  E$. If there exists a symmetric Nash-equilibrium for this game, then it is the unique lexicographically optimal point in $B(M)$ with respect to the weights $1/L_{e,e}$ $(e \in E)$. 
\end{corollary}

The proof follows from observing the partition of edges with respect to the weight vector $Lx$, and proving that symmetric Nash-equilibria satisfy the sufficient conditions for being a lexicographically optimal base. Lexicographically optimal bases can be computed efficiently using \cite{Fujishige1980b}, \cite{Nagano2007}, or the {\sc Inc-fix} algorithm from Section \ref{mirror}. Hence, one could compute the lexicographically optimal base $x$ for a weight vector defined as $w(e) = 1/L_{e,e}$ $(e \in E)$ for a positive diagonal loss matrix $L$, and check if that is a symmetric Nash-equilibrium. If it is, then it is also the unique symmetric Nash-equilibrium and if it is not then there cannot be any other symmetric Nash-equilibrium.


\bibliography{games_colt}

\appendix



\section{Mirror Descent}\label{app:mirror_descent}

\begin{lemma}\label{strong-convexity} The unnormalized entropy map, $\omega(x) = \sum_{i=1}^n x_i \ln x_i - \sum_{i=1}^n x_i$,  is 1-strongly convex with respect to the $L_1$-norm over any matroid base polytope $B(f) = \{x \in \mathbb{R}^{E}: x(E) = r(E), x(E(S)) \leq r(S) \forall S \subseteq E, x \geq 0\}$. 
\end{lemma}
\begin{proof} We have 
\begin{align}
\omega(x) - \omega(y) - \nabla \omega (y)^T (x - y) &= \sum_{e \in E} x_e \ln x_e - \sum_{e\in E} x_e - \sum_{e\in E} y_e \ln y_e + \sum_{e \in E} y_e - \sum_{e \in E} \ln y_e (x_e - y_e)\\
&= \sum_{e \in E} x_e \ln (x_e/y_e) \geq^{(1)} \frac{1}{2}(\sum_{e\in E} |x_e - y_e|)^2 = \frac{1}{2} ||x - y||_1^2,
\end{align}
where (1) follows from Pinsker's inequality. 
\end{proof}

\section{Multiplicative Weights Update Algorithm}\label{app:mwu}

\begin{lemma}\label{regret} Consider an MSP game with strategy polytopes $P \subseteq \mathbb{R}^m$ and $Q \subseteq \mathbb{R}^n$, and let the loss function for the row player be given by $loss(x,y) = x^TLy$ for $x \in P, y \in Q$.  Suppose we simulate an online algorithm {\sc A} such that in each round $t$ the row player chooses decisions from $x^{(t)} \in P$, the column player reveals an adversarial loss vector $v^{(t)}$ such that $x^{(t)T}Lv^{(t)} \geq \max_{y \in Q} x^{(t)T}L y - \delta$ and the row player subsequently incurs loss $x^{(t)T}Lv^{(t)}$ for round $t$. If the regret of the learner after $T$ rounds goes down as $f(T)$, that is,
\begin{align}
R_T(A) = \sum_{i=1}^T x^{(i)T}Lv^{(i)} - \min_{x \in P} \sum_{i=1}^t x^TLv^{(i)} \leq f(T) \label{lem2-regret}
\end{align}
then $(\frac{1}{T}\sum_{i=1}^Tx^{(i)}, \frac{1}{T}\sum_{i=1}^T v^{(i)})$ is an $O(\frac{f(T)}{T} + \delta)$-approximate Nash-equilibrium for the game.  
\end{lemma}
\begin{proof} Let $\bar{x} = \frac{1}{T}\sum_{i=1}^Tx^{(i)}$ and $\bar{v} = \frac{1}{T}\sum_{i=1}^T v^{(i)}$. By the von Neumann minimax theorem, we know that the value of the game is $\lambda^* = \min_{x} \max_{y} x^TLy = \max_{y} \min_{x} x^TLy$. This gives,  
\begin{align}
\min_{x} \max_{y} x^TLy =\lambda^* &\leq \max_{y} \bar{x}^TLy = \max_{y} \frac{1}{T}\sum_{i=1}^T x^{(i)}Ly \leq \frac{1}{T} \sum_{i=1}^T \max_{y} x^{(i)T}L y \label{lem2-eq1}\\
                &\leq \frac{1}{T} (\sum_{i=1}^T x^{(i)T}L v^{(i)} + \delta) \label{lem2-eq2}\\ 
                &\leq \min_{x \in P} \frac{1}{T} \sum_{i=1}^T x^TLv^{(i)} + \frac{f(T)}{T} + \delta \label{lem2-eq3}\\
                &= \min_{x \in P} x^TL \frac{1}{T}\sum_{i=1}^T v^{(i)} + \frac{f(T)}{T} + \delta = \min_{x \in P} x^TL \bar{v} + \frac{f(T)}{T} + \delta \nonumber\\
                &\leq \max_{y \in Q} \min_{x \in P} x^TLy+ \frac{f(T)}{T} +\delta = \lambda^* + \frac{f(T)}{T} + \delta.\nonumber
\end{align} 
where the last inequality in (\ref{lem2-eq1}) follows from the convexity $\max_y x^TLy$ in $x$, (\ref{lem2-eq2}) follows from the error in the adversarial loss vector, and (\ref{lem2-eq3}) follows from the given regret equation (\ref{lem2-regret}). Thus, we get $\bar{x}^TL\bar{v} \leq \max_{y \in Q} \bar{x}^T L y \leq \lambda^* + \frac{f(T)}{T} + \delta$, and $\bar{x}^TL\bar{v} \geq \min_{x \in P} x^TL\bar{v} \geq \lambda^* - \frac{f(T)}{T} - \delta$. Hence, $(\bar{x}, \bar{v})$ is a $\big(\frac{2f(T)}{T}+2\delta\big)$-approximate Nash-equilibrium for the game.  
\end{proof}

\noindent
\textbf{Proof for Theorem \ref{MWU-main-exact}.}
\begin{proof} We want to show that the updates to the weights of each pure strategy $u \in \mathcal{U}$ can be done efficiently. For the multipliers $\lambda^{(t)}$ in each round, let $w^{(t)}(u)$ be the unnormalized probability for each vertex $u$, i.e., ${w^{(t)}}(u) = \prod_{e \in E} \lambda^{(t)}(e)^{u(e)}$ where $E$ is the ground set of elements such that the strategy polytope $P$ of the row player (learner in the MWU algorithm) lies in $\mathbb{R}^{|E|}$. Recall that the set of vertices of $P$ is denoted by $\mathcal{U}$. Let $Z^{(t)}$ be the normalization constant for round $t$, i.e., $Z^{(t)} = \sum_{u \in \mathcal{U}} {w^{(t)}}(u)$. Thus, the probability of each vertex $u$ is $p^{(t)}(u) = {w^{(t)}}(u)/Z^{(t)}$. Let $F = \max_{x \in P, y \in Q} x^TLy$. For readability, we denote the scaled loss $x^TLy/F$ as $\eta(x,y)$; $\eta(x,y) \in [0,1]$.

The algorithm starts with $\lambda^{(1)}(e) = 1$ for all $e \in E$ and thus $w^{(1)}(u)  =1 $ for all $u \in \mathcal{U}$. We claim that $w^{(t+1)}(u) = w^{(t)}(u) \beta^{u^TLv^{(t)}/F}$, where $v^{(t)}$ is revealed to be a vector in $\arg \max_{y \in Q} x^{(t)}Ly$ in each round $t$.
\begin{align}
w^{(t+1)} (u) &= \prod_{e \in E} \lambda^{(t+1)}(e)^{u(e)} = \prod_{e \in E} \lambda^{(t)}(e)^{u(e)} \beta^{(L v^{(t)}/F)_e}\nonumber\\
&= \beta^{(u^{T} L v^{(t)}/F)}\prod_{e \in E}  \lambda^{(t)}(e)^{u(e)} && \hdots u \in \{0,1\}^m.\nonumber\\
&= {w^{(t)}}(u) \beta^{u^TLv^{(t)}/F} = {w^{(t)}}(u) \beta^{\eta(u, v^{(t)})}. \nonumber
\end{align}

Having shown that the multiplicative update can be performed efficiently, the rest of the proof follows as standard proofs the MWU algorithm. We prove that $Z^{(t+1)} \leq Z^{(1)}  \exp (- (1-\beta) \sum_{i=1}^t \eta(x^{(i)}, v^{(i)})).$ 
\begin{align*}
Z^{(t+1)} &= \sum_{u \in \mathcal{U}} w^{(t+1)}(u) = \sum_{u \in \mathcal{U}} {w^{(t)}}(u) \beta^{\eta(u, v^{(t)})} \\
&\leq \sum_{u \in \mathcal{U}} {w^{(t)}}(u) (1 - (1-\beta)\eta(u, v^{(t)})) && \hdots (1-\theta)^x \leq 1-\theta x \textrm{ for } x \in [0,1], \theta \in [-1,1].\\
&= Z^{(t)} (1   -  (1-\beta) \eta(x^{(t)}, v^{(t)}))\\
&\leq Z^{(t)} \exp (- (1-\beta) \eta(x^{(t)}, v^{(t)})) && \hdots (1 - \theta x) \leq e^{-x\theta} ~\forall ~x, ~\forall~ \theta. 
\end{align*}
Rolling out the above till the first round, we get 
\begin{align}
Z^{(t+1)} &\leq Z^{(1)} \exp (- (1-\beta)\sum_{i=1}^t \eta(x^{(i)}, v^{(i)})). \label{boundZt}
\end{align}
Since $w^{(t+1)} (u) \leq Z^{(t+1)} \textrm{ for all } u \in \mathcal{U}$, using (\ref{boundZt}), we get 
\begin{align*}
\ln w^{(1)}(u) + \ln \beta \sum_{i=1}^{t} \eta(u, v^{(i)}) &\leq \ln Z^{(1)}  - (1-\beta) \sum_{i=1}^t \eta(x^{(i)},v^{(i)}) \\
\Rightarrow   \sum_{i=1}^t \eta(x^{(i)},v^{(i)})  &\leq  - \frac{\ln \beta}{(1-\beta)} \sum_{i=1}^t  \eta(u, v^{(i)}) + \frac{\ln Z^{(1)} - \ln w^{(1)}(u)}{1-\beta} && \hdots \beta < 1
\end{align*}
\begin{align*}
\Rightarrow   \sum_{i=1}^t \eta(x^{(i)},v^{(i)})  &\leq  \frac{1+ \beta}{2\beta} \sum_{i=1}^t  \eta(u, v^{(i)}) + \frac{\ln Z^{(1)} - \ln w^{(1)}(u)}{1-\beta} \\
& \hdots \frac{-\ln x}{1- x} \leq \frac{1+x}{2x} \textrm{ for } x \in (0,1]
\end{align*}
\begin{align*}
\Rightarrow   \frac{1}{t}\sum_{i=1}^t \eta(x^{(i)},v^{(i)}) &\leq  \frac{1+ \beta}{2\beta t} \sum_{i=1}^t  \eta(u, v^{(i)})+ \frac{\ln Z^{(1)} - \ln w^{(1)}(u)}{(1-\beta)t}
\end{align*}
Consider an arbitrary mixed strategy for the row player $\mathbf{p}$ such that $\sum_{u \in \mathcal{U}} \mathbf{p}(u)u = \mathbf{x}$, and multiply the above equation for each vertex $u$ with the $\mathbf{p}(u)$ and sum them up - 
\begin{align*}
\frac{1}{t} \displaystyle\sum_{i=1}^t   \eta(x^{(i)},v^{(i)})  &\leq \frac{1+\beta}{2\beta t}\displaystyle\sum_{i=1}^t \eta(\mathbf{x}, v^{(i)}) + \frac{\ln Z_1 - \ln w^{(1)}(u)}{(1-\beta)t}
\end{align*}
Note that $Z_1 = |\mathcal{U}|$, $w^{(1)}(u) = 1$. Setting $\epsilon^\prime = \epsilon/F$, $\beta = \frac{1}{1+\sqrt{2} \epsilon^\prime}$, $t = \frac{\ln |\mathcal{U}|}{\epsilon^{\prime2}} = \frac{F^2\ln(|\mathcal{U}|)}{\epsilon^2}$ we get 
\begin{align}
\frac{1}{t} \displaystyle\sum_{i=1}^t   \eta(x^{(i)},v^{(i)})  &\leq \frac{2+\sqrt{2}\epsilon^\prime}{2 t}\displaystyle\sum_{i=1}^t \eta(\mathbf{x}, v^{(i)}) + \frac{\epsilon^{\prime2}\ln (|\mathcal{U}|) (1+\sqrt{2}\epsilon^\prime)}{\sqrt{2} \epsilon^{\prime} \ln (|\mathcal{U}|)}\\
\Rightarrow \frac{1}{t} \displaystyle\sum_{i=1}^t   \eta(x^{(i)},v^{(i)}) &\leq \frac{1}{t}\sum_{i=1}^t \eta(\mathbf{x},v^{(i)}) + \sqrt{2}\epsilon^\prime + \epsilon^{\prime 2} && \hdots \textrm{using } \sum_{i=1}^t \eta(\mathbf{x}, v^{(i)}) \leq t\\
\Rightarrow \frac{1}{t} \displaystyle\sum_{i=1}^t x^{(i)T}L v^{(i)}  &\leq \frac{1}{t}\sum_{i=1}^t \mathbf{x}^TLv^{(i)} + O(\epsilon) \label{poly-regret}
\end{align}

Thus, we get that the MWU algorithm converges to an $\epsilon$-approximate Nash-equilibrium in $O(F^2\ln(|\mathcal{U}|)/\epsilon^2)$ rounds.
 
\end{proof}

\noindent
\textbf{Proof for Lemma \ref{MWU-main}.}
\begin{proof} Let the multipliers in each round $t$ be $\lambda^{(t)}$, the corresponding true and approximate marginal points in round $t$ be $x^{(t)}$ and $\tilde{x}^{(t)}$ respectively, such that $||x^{(t)} - \tilde{x}^{(t)}||_{\infty} \leq \epsilon_1$. Now, we can only compute approximately adversarial loss vectors, $\tilde{v}^{(t)}$ in each round such that $\tilde{x}^{(t)}L\tilde{v}^{(t)} \geq \max_{y \in Q} \tilde{x}^{(t)} L y - \epsilon_2$. 

Even though we cannot compute $x^{(t)}$ exactly, we do maintain the corresponding $\lambda^{(t)}$s that correspond to these true marginals. Let us analyze first the multiplicative updates corresponding to approximate loss vectors $\tilde{v}^{(t)}$ using product distributions (that can be done efficiently) over the true marginal points. Using the proof for Theorem \ref{MWU-main-exact}, we get the following regret bound with respect to the true marginals corresponding to (\ref{poly-regret}) for $t = \frac{F^2 \ln(|\mathcal{U}|)}{\epsilon^2}$ rounds:

\begin{align} \frac{1}{t} \displaystyle\sum_{i=1}^t x^{(i)T}L \tilde{v}^{(i)}  &\leq \frac{1}{t}\sum_{i=1}^t \mathbf{x}^TL\tilde{v}^{(i)} + O(\epsilon) \label{regret-approx}
\end{align}

We do not have the value for $x^{(i)}$ for $i = 1, \hdots, t$, but only estimates $\tilde{x}^{(i)}$ for $i = 1, \hdots, t$ such that $||\tilde{x}^{(i)} - x^{(i)}||_\infty \leq \epsilon_1$. Since the losses we consider are bilinear, we can bound the loss of the estimated point in each iteration $i$ as follows:
\begin{align}
| \tilde{x}^{(i)T}L\tilde{v}^{(i)} - x^{(i)T}L \tilde{v}^{(i)} | & \leq \epsilon_1 e L \tilde{v}^{(i)} \leq F\epsilon_1,
\end{align}
where $e = (1,\hdots, 1)^T$. Thus using (\ref{regret-approx}), we get
\begin{align}
\frac{1}{t} \displaystyle\sum_{i=1}^t \tilde{x}^{(i)T}L \tilde{v}^{(i)} &\leq \frac{1}{t}\sum_{i=1}^t \mathbf{x}^TL\tilde{v}^{(i)} + O(\epsilon + F\epsilon_1).
\end{align}
Now, considering that we played points $\tilde{x}^{(i)}$ for each round $i$, and suffered losses $\tilde{v}^{(i)}$, we have shown that the {\sc MWU} algorithm achieves $O(\epsilon +  F\epsilon_1)$ regret on an average. Thus, as $\mathbf{R}_{\epsilon_2}$ is assumed to have error $\epsilon_2$, using Lemma \ref{regret} we have that $(\frac{1}{t}\sum_{i=1}^t\tilde{x}^{(i)}, \frac{1}{t}\sum_{i=1}^t\tilde{v}^{(i)})$ is an $O(\epsilon +  F\epsilon_1 + \epsilon_2)$-approximate Nash-equilibrium. 
\end{proof}

\end{document}